\documentclass{article} %

\usepackage[T1]{fontenc}    %
\usepackage[utf8]{inputenc} %
\usepackage[english]{babel}
\usepackage{amsmath, amssymb, amsthm}
\usepackage{hyperref}       %
\usepackage{url}            %
\usepackage{booktabs}       %
\usepackage{amsfonts}       %
\usepackage{nicefrac}       %
\usepackage{microtype}      %
\usepackage{natbib}
\usepackage{enumitem}
\usepackage{soul}
\usepackage{siunitx}
\usepackage{caption}
\usepackage{subcaption}
\usepackage{tikz}
\usepackage{graphicx}
\usepackage{multirow}
\usepackage[makeroom]{cancel}
\usepackage{xargs}  
\usepackage{xcolor}
\usepackage[colorinlistoftodos, textsize=tiny]{todonotes}
\usepackage{afterpage}

\usepackage{tikzit}
\tikzstyle{dotted circle}=[tikzit category=nodes, shape=circle, draw=black, minimum size=0.6cm, dashed, thick]
\tikzstyle{white node}=[fill=white, tikzit category=nodes, shape=circle, draw=black, minimum size=0.6cm, thick]
\tikzstyle{black node}=[fill=black, tikzit category=nodes, shape=circle, draw=black, minimum size=0.6cm, thick]
\tikzstyle{red node}=[fill=red, tikzit category=nodes, shape=circle, draw=black, minimum size=0.6cm, thick]
\tikzstyle{blue node}=[fill=blue, shape=circle, draw=black, tikzit category=nodes, minimum size=0.6cm, thick]
\tikzstyle{green node}=[tikzit fill=green, fill=green, shape=circle, draw=black, tikzit category=nodes, minimum size=0.6cm, thick]
\tikzstyle{yellow square}=[draw=black, fill=yellow, shape=rectangle, minimum size=0.6cm, thick]
\tikzstyle{blue node 2}=[fill={rgb,255: red,128; green,0; blue,128}, draw=black, shape=circle, tikzit fill=blue, minimum size=0.6cm, thick]

\tikzstyle{arrow}=[->, thick]
\tikzstyle{dashed arrow}=[->, dashed, thick]
\tikzstyle{dashed edge}=[<->, dashed, thick]
\tikzstyle{blue pointer}=[->, draw=blue, thick]
\tikzstyle{red pointer}=[->, draw=red, thick]
\tikzstyle{thick blue pointer}=[->, draw=blue, very thick]
\tikzstyle{thick red pointer}=[->, draw=red, very thick]
\tikzstyle{thick pointer}=[->, very thick]
\tikzstyle{thick line}=[-, very thick]

\usepackage{amsmath,amsfonts,bm}

\def\eqref#1{equation~\ref{#1}}

\def\1{\bm{1}}

\def\re{{\textnormal{e}}}

\DeclareMathAlphabet{\mathsfit}{\encodingdefault}{\sfdefault}{m}{sl}
\SetMathAlphabet{\mathsfit}{bold}{\encodingdefault}{\sfdefault}{bx}{n}

\def\gN{{\mathcal{N}}}

\newcommand{\E}{\mathbb{E}}

\newcommand{\R}{\mathbb{R}}

\newcommand{\Var}{\mathrm{Var}}

\DeclareMathOperator*{\argmin}{arg\,min}

\usepackage[accepted]{icml2020}
\usepackage{times}

\usepackage{defs}

\setlist[itemize]{noitemsep,topsep=-6pt}
\setlist[enumerate]{noitemsep,topsep=-6pt}

\usepackage{thm-restate}

\begin{document}

\twocolumn[
\icmltitle{
AR-DAE: Towards Unbiased Neural Entropy Gradient Estimation
}

\icmlsetsymbol{equal}{*}

\begin{icmlauthorlist}
\icmlauthor{Jae Hyun Lim}{mila,udem}
\icmlauthor{Aaron Courville}{mila,udem,cifar,ccai}
\icmlauthor{Christopher Pal}{mila,poly,ccai}
\icmlauthor{Chin-Wei Huang}{mila,udem}
\end{icmlauthorlist}

\icmlaffiliation{mila}{Mila}
\icmlaffiliation{poly}{Polytechnique Montr\'eal}
\icmlaffiliation{udem}{Universit\'e de Montr\'eal}
\icmlaffiliation{cifar}{CIFAR fellow}
\icmlaffiliation{ccai}{Canada CIFAR AI Chair}

\icmlcorrespondingauthor{Jae Hyun Lim}{jae.hyun.lim@umontreal.ca}
\icmlcorrespondingauthor{Chin-Wei Huang}{chin-wei.huang@umontreal.ca}

\icmlkeywords{deep auto-encoders, information theory and estimation, variational inference, likelihood-free inference, generative models, reinforcement learning, maximum-entropy}

\vskip 0.3in
]

\printAffiliationsAndNotice{}  %

\begin{abstract}
Entropy is ubiquitous in machine learning, %
but it is in general intractable to compute the entropy of the distribution of an arbitrary continuous random variable.
In this paper, we propose the \emph{amortized residual denoising autoencoder} (AR-DAE) %
to approximate the gradient of the log density function, which can be used to estimate the gradient of entropy.
Amortization allows us to 
significantly reduce the error of the gradient approximator
by approaching asymptotic optimality of a regular DAE, in which case the estimation is in theory unbiased. %
We conduct theoretical and experimental analyses on the approximation error of the proposed method, as well as extensive studies on heuristics to ensure its robustness.~Finally, using the proposed gradient approximator to estimate the gradient of entropy, we demonstrate state-of-the-art performance on density estimation with variational autoencoders and continuous control with soft actor-critic.
\end{abstract}

\section{Introduction}
\label{sec:introduction}
Entropy is an information theoretic measurement of uncertainty that has found many applications in machine learning. 
For example, it can be used to incentivize exploration in reinforcement learning (RL) \citep{HaarnojaTAL17/icml, HaarnojaZAL18/icml}; prevent mode-collapse of generative adversarial networks (GANs) \citep{BalajiHCF19/icml, Dieng2019pregan}; and calibrate the uncertainty of the variational distribution in approximate Bayesian inference. 
However, it is in general intractable to compute the entropy of an arbitrary random variable.

In most applications, one actually does not care about the quantity of entropy itself, but rather how to manipulate and control this quantity as part of the optimization objective. 
In light of this, we propose to approximately estimate the gradient of entropy so as to maximize or minimize the entropy of a data sampler. 
More concretely, we approximate the gradient of the log probability density function of the data sampler.
This is sufficient since the gradient of its entropy can be shown to be the expected value of the \textit{path derivative}~\citep{roeder2017sticking}. 
We can then plug in a gradient approximator to enable stochastic backpropagation.

We propose to use the \emph{denoising autoencoder} (DAE, \citet{VincentLBM08/icml}) to approximate the gradient of the log density function, which is also known as \emph{denoising score matching} \citep{vincent2011connection}. 
It has been shown that the optimal reconstruction function of the DAE converges to the gradient of the log density as the noise level $\sigma$ approaches zero \citep{AlainB14/jmlr}. 
In fact, such an approach has been successfully applied to recover the gradient field of the density function of high-dimensional data such as natural images \citep{song2019generative}, which convincingly shows DAEs can accurately approximate the gradient. 
However, in the case of entropy maximization (or minimization), 
the non-stationarity of the sampler's distribution poses a problem for optimization.
On the one hand, the log density gradient is recovered only asymptotically as $\sigma\rightarrow0$.
On the other hand, the training signal vanishes while a smaller noise perturbation is applied, which makes it hard to reduce the approximation error due to suboptimal optimization. 
The fact that the sampler's distribution is changing makes it even harder to select a noise level that is sufficiently small.
Our work aims at resolving this no-win situation. 

In this work, we propose the \emph{amortized residual denoising autoencoder} (AR-DAE), which is a conditional DAE of a residual form that takes in $\sigma$ as input. 
We condition the DAE on $\sigma=0$ at inference time to approximate the log density gradient while sampling non-zero $\sigma$ at training, which allows us to train with $\sigma$ sampled from a distribution that covers a wide range of values. 
If AR-DAE is optimal, we expect to continuously generalize to $\sigma=0$ to recover the log density gradient, which can be used as an unbiased estimate of the entropy gradient. 
We perform ablation studies on the approximation error using a DAE, and show that our method provides significantly more accurate approximation than the baselines. 
Finally, we apply our method to improve distribution-free inference for variational autoencoders \citep{kingma2013auto,rezende2014stochastic} and soft actor-critic \citep{HaarnojaZAL18/icml} for continuous control problems in reinforcement learning. 
As these tasks are non-stationary, amortized (conditional), and highly structured, it demonstrates AR-DAE can robustly and accurately approximate log density gradient of non-trivial distributions given limited computational budgets.

\section{Approximate entropy gradient estimation}

\subsection{Background on tractability of entropy}
\label{para:implicit-density-models}
An \emph{implicit density model} is characterized by a data generation process \citep{%
Mohamed2016learning}.
The simplest form of an implicit density model contains a prior random variable $z\sim p(z)$, and a generator function $g:z\mapsto x$.
The likelihood of a particular realization of $x$ is \emph{implied} by the pushforward of $p(z)$ through the mapping $g$. 

Unlike an \emph{explicit density model}, 
an implicit density model does not require a carefully designed parameterization for the density to be explicitly defined, 
allowing it to approximate arbitrary data generation process more easily. %
This comes at a price, though, since the density function of the implicit model cannot be easily computed, 
which makes it hard to approximate its entropy using Monte Carlo methods.

\subsection{Denoising entropy gradient estimator}

\begin{figure}%
\centering
\begin{subfigure}[b]{0.5\textwidth}
    \centering
    \scalebox{0.80}{
    \begin{tikzpicture}
	\begin{pgfonlayer}{nodelayer}
		\node [style=none] (44) at (3.5, 0.325) {$\mathbb{E}_{x} [ - \log p_{\theta}(x) ] =$};
		\node [style=none] (17) at (3.5, -0.325) {\small $\mathbb{E}_{z} \Big[ \log \textrm{det} | \mathbf{J}_{z} g_{\theta}(z) | - \log p(z)\Big]$};
		\node [style=none] (19) at (-3.25, -0.5) {$\theta$};
		\node [style=none] (21) at (1.25, 0.5) {};
		\node [style=none] (22) at (1.25, -0.5) {};
		\node [style=none] (35) at (-2.5, 0.5) {};
		\node [style=none] (36) at (-2.5, -0.5) {};
		\node [style=none] (37) at (-0.6, -0.9) {$\mathbb{E}_{z} [ \nabla_{\theta} \log \textrm{det} | \mathbf{J}_{z} g_{\theta}(z) |]$};
		\node [style=none] (38) at (-0.6, 0.9) {$\log \textrm{det} | \mathbf{J}_{z} g_{\theta}(z) |$};
		\node [style=none] (39) at (-0.6, 1.4) {$x = g_{\theta}(z)$};
		\node [style=none] (41) at (-3.25, 0.5) {$z \sim p(z)$};
		\node [style=none, text=blue] (42) at (-0.75, 0.28) {\small $\tt{forward}$};
		\node [style=none, text=red] (43) at (-0.75, -0.28) {\small $\tt{backward}$};
	\end{pgfonlayer}
	\begin{pgfonlayer}{edgelayer}
		\draw [style=thick blue pointer] (35.center) to (21.center);
		\draw [style=thick red pointer] (22.center) to (36.center);
	\end{pgfonlayer}
\end{tikzpicture}
    }
    \caption{}
\end{subfigure}

\vspace*{-0.2cm}

\hspace*{0.1cm}

\begin{subfigure}[b]{0.5\textwidth}
    \centering
    \scalebox{0.80}{
    \begin{tikzpicture}
	\begin{pgfonlayer}{nodelayer}
		\node [style=none] (19) at (-3.25, -0.5) {$\theta$};
		\node [style=none] (21) at (1.25, 0.5) {};
		\node [style=none] (22) at (1.25, -0.5) {};
		\node [style=none] (35) at (-2.5, 0.5) {};
		\node [style=none] (36) at (-2.5, -0.5) {};
		\node [style=none] (37) at (-0.6, -0.9) {- $\mathbb{E}_{z} [f_{ar}(g_{\theta}(z))^{\intercal} \mathbf{J}_{\theta} g_{\theta}(z) |]$};
		\node [style=none] (38) at (-0.6, 0.9) {$x = g_{\theta}(z)$};
		\node [style=none] (41) at (-3.25, 0.5) {$z \sim p(z)$};
		\node [style=none, text=blue] (42) at (-0.75, 0.28) {\small $\tt{forward}$};
		\node [style=none, text=red] (43) at (-0.75, -0.28) {\small $\tt{backward}$};
		\node [style=none] (45) at (1.5, 0) {$x$};
		\node [style=none] (48) at (2.5, 0) {};
		\node [style=none] (49) at (1.75, 0.25) {};
		\node [style=none] (50) at (1.75, -0.25) {};
		\node [style=none] (51) at (4.15, -0.25) {$f_{ar}(x) \approx \nabla_{x} \log p_{\theta}(x)$};
		\node [style=none] (52) at (4.15, 0.25) {AR-DAE};
	\end{pgfonlayer}
	\begin{pgfonlayer}{edgelayer}
		\draw [style=thick blue pointer] (35.center) to (21.center);
		\draw [style=thick red pointer] (22.center) to (36.center);
		\draw [style=thick line, bend left=60, looseness=1.25] (49.center) to (48.center);
		\draw [style=thick pointer, bend left=75, looseness=1.50] (48.center) to (50.center);
	\end{pgfonlayer}
\end{tikzpicture}
    }
    \caption{}
\end{subfigure}

\vspace*{-0.2cm}

\caption{
(a) Entropy gradient wrt parameters of an invertible generator function. 
(b) Approximate entropy gradient using the proposed method.
}
\label{fig:explicit-vs-implicit}

\vspace*{-0.3cm}

\end{figure}
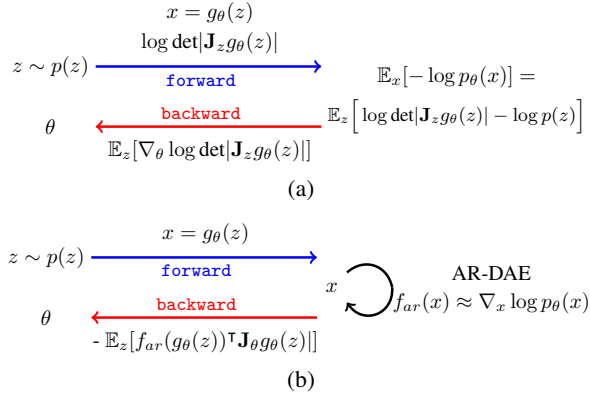

Let $z$ and $g$ be defined as above, and let $\theta$ be the parameters of the mapping $g$ (denoted $g_\theta$). 
Most of the time, we are interested in maximizing (or minimizing) the entropy of the implicit distribution of $x=g_\theta(z)$. 
For example, when the mapping $g$ is a bijection, the density of $x=g_\theta(z)$ can be decomposed using the change-of-variable density formula, so controlling the entropy of $x$ amounts to controlling the log-determinant of the Jacobian of $g_\theta$ \citep{RezendeM15/icml}, as illustrated in Figure~\ref{fig:explicit-vs-implicit}-(a).
This allows us to estimate both the entropy and its gradient. 
However, for an iterative optimization algorithm such as (stochastic) gradient descent, which is commonly employed in machine learning, it is sufficient to compute the gradient of the entropy rather than the entropy itself.  %

Following \citet{roeder2017sticking}, we can rewrite the entropy of $x$ by changing the variable and neglecting the score function which is $0$ in expectation to get
\eq{
\label{eq:entropy-gradient}
\nabla_{\theta} H(p_{g}(x)) = -\E_{z} \left[ [\nabla_x \log p_{g}(x)|_{x=g_{\theta}(z)}]^{\intercal} \mathbf{J}_{\theta}g_{\theta}(z) \right] %
,
}
where $\mathbf{J}_{\theta}g_{\theta}(z)$ is the Jacobian matrix of the random sample $x=g_{\theta}(z)$ wrt to the sampler's parameters $\theta$.
See Appendix \ref{sec:appendix-entropy-gradient} for the detailed derivation.
We emphasize that this formulation is more general as it does not require $g$ to be bijective. 

Equation (\ref{eq:entropy-gradient}) tells us that we can obtain an unbiased estimate of the entropy by drawing a sample of the integrand, which is the \emph{path derivative} of $z$.
The integrand requires evaluating the sample $x=g_\theta(z)$ under the gradient of its log density $\nabla_x \log p_g(x)$. 
As $\log p_g(x)$ is usually intractable or simply not available, we directly approximate its gradient using a black box function. 
As long as we can provide a good enough approximation to the gradient of the log density and treat it as the incoming unit in the backward differentiation (see Figure \ref{fig:explicit-vs-implicit}-(b)), the resulting estimation of the entropy gradient is approximately unbiased. 

In this work, we propose to approximate the gradient of the log density using a denoising autoencoder (DAE, \citet{VincentLBM08/icml}). 
A DAE is trained by minimizing the reconstruction loss $d$ of an autoencoder $r$ with a randomly perturbed input
\begin{align*}
\mathcal{L}_{\tt{DAE}}(r)=\E[d(x, r(x+\epsilon))],
\end{align*}
where the expectation is taken over the random perturbation $\epsilon$ and data $x$. 
\citet{AlainB14/jmlr} showed that if $d$ is the L2 loss and $\epsilon$ is a centered isotropic Gaussian random variable with variance $\sigma^2$, then under some mild regularity condition on $\log p_g$ the optimal reconstruction function satisfies
$$r^*(x)=x+\sigma^2\nabla_x \log p_g(x) + o(\sigma^2),$$
as $\sigma^2\rightarrow0$.
That is, for sufficiently small $\sigma$, we can approximate the gradient of the log density using the black box function $f_r(x):=\frac{r(x)-x}{\sigma^2}$ assuming $r\approx r^*$.

\section{Error analysis of $\nabla_x\log p_g(x)\approx f_r(x)$}
Naively using $f_r(x)$ to estimate the gradient of the entropy is problematic. 
First of all, the division form of $f_r$ can lead to numerical instability and magnify the error of approximation.
This is because when the noise perturbation $\sigma$ is small, $r(x)$ will be very close to $x$ and thus both the numerator and the denominator of $f_r$ are close to zero. 

Second, using the triangle inequality, we can decompose the error of the approximation $\nabla_x\log p_g(x)\approx f_r(x)$ into
\begin{align*}
||\nabla_x &\log p_g(x) - f_r(x)|| \leq \\ &\!\!\qquad\underbrace{||\nabla_x\log p_g(x) - f_{r^*}(x)||}_{\textit{asymp error}} \;\,+\;\, ||f_{r^*}(x) - f_{r}(x) ||.
\end{align*}
The first error is incurred by using the optimal DAE to approximate $\nabla_x\log p_g(x)$, which vanishes when $\sigma\rightarrow 0$.
We refer to it as the \emph{asymptotic error}.
The second term is the difference between the optimal DAE and the ``current'' reconstruction function. 
Since we use a parametric family of functions (denoted by $\mathcal{F}$) to approximate $f_{r^*}$, it can be further bounded by
\begin{align*}
||f_{r^*}&(x) - f_{r}(x) || \leq \\
&\!\!\qquad\qquad\underbrace{||f_{r^*}(x) - f_{r_{\mathcal{F}}^*}(x) ||}_{\textit{param error}} \;\,+\;\, \underbrace{||f_{r_{\mathcal{F}}^*}(x) - f_{r}(x) ||}_{\textit{optim error}},
\end{align*}
where ${r_{\mathcal{F}}^*}:=\argmin_{r\in\mathcal{F}} \cL_{\tt{DAE}}(r)$ is the optimal reconstruction function within the family $\mathcal{F}$. 
The first term measures how closely the family of functions $\mathcal{F}$ approximates the optimal DAE, and is referred to as the \emph{parameterization error}. 
The second term reflects the suboptimality in optimizing $r$. 
It can be significant especially when the distribution of $x$ is non-stationary, in which case $r$ needs to be constantly adapted. 
We refer to this last error term as the \emph{optimization error}.
As we use a neural network to parameterize $r$, the parameterization error can be reduced by increasing the capacity of the network. 
The optimization error is subject to the %
variance of the noise $\sigma^2$ (relative to the distribution of $x$), as it affects the magnitude of the gradient signal $\E[\nabla||r(x+\epsilon)-x||^2]$.
This will make it hard to design a fixed training procedure for $r$ as different values of $\sigma$ requires different optimization specifications to tackle the optimization error.

\section{Achieving asymptotic optimality}
In this section, we propose 
the \textbf{amortized residual DAE} (AR-DAE), 
an improved method to approximate $\nabla_x\log p_g(x)$ that is designed to resolve the numerical instability issue and reduce the error of approximation. 

\subsection{Amortized residual DAE}
\label{subsec:ardae}
AR-DAE  (denoted $f_{ar}$) is a DAE of residual form conditioned on the magnitude of the injected noise, 
minimizing the following optimization objective.
\begin{align}
\label{eq:ardae}
\cL_{\tt{ar}}\lbp f_{ar} \rbp
= \eE_{\substack{
        x \sim p(x) \\
        u \sim N(0, I) \\
        \sigma \sim N(0, \delta^2) \\
        }}
    \lbs \lbV u + \sigma f_{ar}(x + \sigma u; \sigma) \rbV^2 \rbs   
.
\end{align}
This objective involves three modifications to the regular training and parameterization of a DAE: \emph{residual connection}, \emph{loss rescaling}, and \emph{scale conditioning} for amortization. 

\para{Residual form} 
First, 
we consider a residual form of DAE
(up to a scaling factor): 
let
$r(x)=\sigma^2 f_{ar}(x)+x$, 
then
$\nabla_x \log p_g(x)$ is approximately equal to
$$\frac{r(x)-x}{\sigma^2} = \frac{\sigma^2 f_{ar}(x)+x -x}{\sigma^2} = f_{ar} \textrm{ .}$$
That is, this reparameterization allows $f_{ar}$ to directly approximate the gradient, avoiding the division that can cause numerical instability.
The residual form also has an obvious benefit of a higher capacity, as it allows the network to represent an identity mapping more easily, which is especially important when the reconstruction function is close to an identity map for small values of $\sigma$  \cite{HeZRS16/cvpr}.%

\para{Loss rescaling} To prevent the gradient signal from vanishing to $0$ too fast when $\sigma$ is arbitrarily small,
we rescale the loss $\cL_{\tt{DAE}}$ by a factor of $1/\sigma$, and since we can decouple the noise level from the isotropic Gaussian noise into $\epsilon=\sigma u$ for standard Gaussian $u$, the rescaled loss can be written as $\E[||\sigma f_{ar}(x+\sigma u)+u||^2]$.

We summarize the properties of the optimal DAE of the rescaled residual form in the following propositions:
\begin{restatable}{prop}{optimalres}
\label{prop:optimal_res}
Let $x$ and $u$ be distributed by $p(x)$ and $\cN(0,I)$. For $\sigma\neq0$, the minimizer of the functional $\E_{x,u}[||u+\sigma f(x+\sigma u)||^2]$
is almost everywhere determined by
$$f^*(x;\sigma) = \frac{-\eE_u[p(x-\sigma u)u]}{\sigma \eE_u[p(x-\sigma u)]} \textrm{ .}$$
Furthermore, if $p(x)$ and its gradient are both bounded, $f^*$ is continuous wrt $\sigma$ for all $\sigma\in\R\setminus0$ and $\lim_{\sigma\rightarrow0}f^*(x;\sigma) = \nabla_x \log p_g(x)$.
\end{restatable}

The above proposition studies the asymptotic behaviour of the optimal $f^*_{ar}$ as $\sigma\rightarrow0$. 
Below, we show that under the same condition, $f^*_{ar}$ approaches the gradient of the log density function of a Gaussian distribution centered at the expected value of $x\sim p(x)$ as $\sigma$ is arbitrarily large. 
\begin{restatable}{prop}{reslargesigma}
\label{prop:res_large_sigma}
$\lim_{\sigma\rightarrow\infty}\frac{f^*(x;\sigma)}{\nabla_x \log\cN(x;\E_p[X],\sigma^2I)} \rightarrow1$.
\end{restatable}

\begin{figure}
\centering
\includegraphics[width=0.23\textwidth]{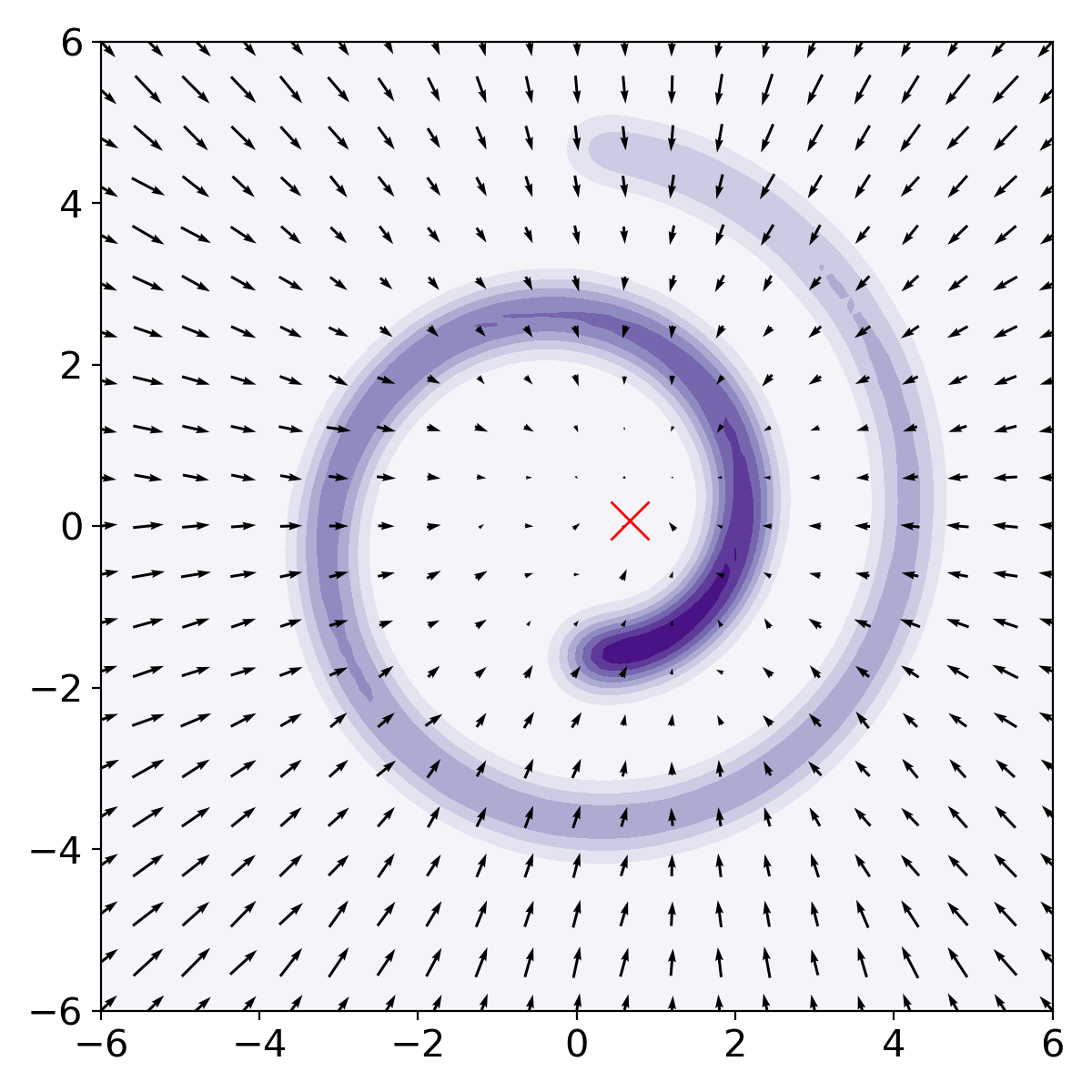}
\hfill
\includegraphics[width=0.23\textwidth]{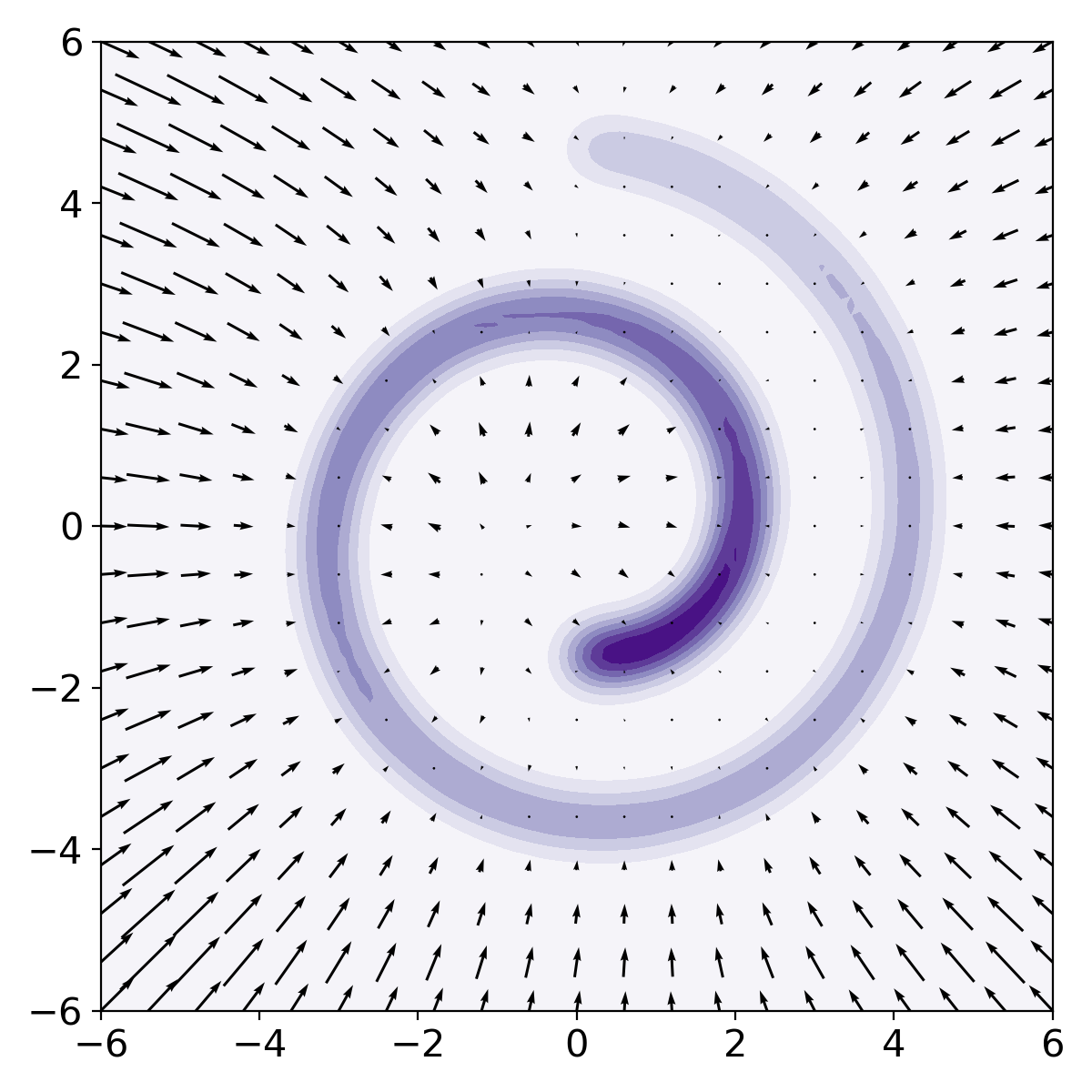}
\caption{Residual DAE trained with a large (left) vs small (right) $\sigma$ value. Red cross indicates the mean of the swissroll. 
The arrows indicate the approximate gradient directions.}
\label{fig:quivers}

\vspace*{-0.3cm}

\end{figure}

\para{Scale conditioning} Intuitively, with larger $\sigma$ values, the perturbed data $x+\sigma u$ will more likely be ``off-manifold'', which makes it easy for the reconstruction function to point back to where most of the probability mass of the distribution of $x$ resides.
Indeed, as Proposition~\ref{prop:res_large_sigma} predicts, with larger $\sigma$ the optimal $f_{ar}^*$ tends to point to the expected value $\E_p[X]$, which is shown in Figure \ref{fig:quivers}-left. 
With smaller values of $\sigma$, training $f_{ar}$ becomes harder, as one has to predict the vector $-u$ from $x+\sigma u$ (i.e. treating $x$ as noise and trying to recover $u$).
Formally, the training signal ($\Delta$) has a decaying rate of $\mathcal{O}(\sigma^2)$ for small $\sigma$ values, because
\begin{align*}%
\E_u[\Delta] &:= \E_u[\nabla||u+\sigma f(x+\sigma u)||^2] \nonumber \\&= 2\sigma^2 \nabla \left(\tr(\nabla_x f(x)) + \frac{1}{2}||f(x)||^2 \right) + o(\sigma^2),
\end{align*}%
where the first term is proportional to the stochastic gradient of the \emph{implicit score matching} \citep{hyvarinen2005estimation}.
That is, with smaller $\sigma$ values, minimizing the rescaled loss is equivalent to score matching, up to a diminishing scaling factor. 
Moreover, the variance of the gradient signal $\Var(\Delta)$ also has a quadratic rate $\mathcal{O}(\sigma^2)$, giving rise to a decreasing signal-to-noise ratio (SNR) $\E[\Delta]/\sqrt{\Var(\Delta)}=\mathcal{O}(\sigma)$, which is an obstacle for stochastic optimization \citep{shalev2017failures}.
See Appendix \ref{app:snr} for the SNR analysis.

In order to leverage the asymptotic optimality of the gradient approximation as $\sigma\rightarrow0$ (Figure \ref{fig:quivers}-right),
we propose to train multiple (essentially infinitely many) models with different $\sigma$'s at the same time, hoping to leverage the benefit of training a large-$\sigma$ model while training a model with a smaller $\sigma$. 

More concretely, we condition $f_{ar}$ on the scaling factor $\sigma$, so that $f_{ar}$ can "generalize" to the limiting behaviour of $f^*_{ar}$ as $\sigma\rightarrow0$ to reduce the asymptotic error.
Note that we cannot simply take $\sigma$ to be zero, since setting $\sigma=0$ would result in either learning an identity function for a regular DAE or learning an arbitrary function for the rescaled residual DAE (as the square loss would be independent of the gradient approximator). 

The scale-conditional gradient approximator $f_{ar}(x;\sigma)$ will be used to approximate $\nabla_x \log p_g(x)$ by setting $\sigma=0$ during inference, while $\sigma$ is never zero at training. 
This can be done by considering a distribution of $\sigma$, which places zero probability to the event $\{\sigma=0\}$; e.g. a uniform density between $[0,\delta]$ for some $\delta>0$. 
The issue of having a non-negative support for the distribution of $\sigma$ is that we need to rely on $f_{ar}$ to extrapolate to $0$, but neural networks usually perform poorly at extrapolation. 
This can be resolved by having a symmetric distribution such as centered Gaussian with variance $\delta^2$ or uniform density between $[-\delta,\delta]$; owing to the the symmetry of the noise distribution $N(u;0,I)$, we can mirror the scale across zero without changing the loss:
\begingroup\makeatletter\def\f@size{9}\check@mathfonts
\def\maketag@@@#1{\hbox{\m@th\large\normalfont#1}}%
\begin{align*}
\E_u
\lbs \lbV u + \sigma f(x+\sigma u) \rbV^2 \rbs \nn
&= \eE \lbs \lbV (-u) + \sigma f(x+\sigma (-u)) \rbV^2 \rbs \nn \\
&= \eE \lbs \lbV u + (-\sigma) f(x+(-\sigma)u) \rbV^2 \rbs \nn.%
\end{align*}\endgroup
Furthermore, Proposition \ref{prop:optimal_res} implies a good approximation to $f_{ar}^*(x,\sigma')$ would be close to $f_{ar}^*(x,\sigma)$ if $\sigma'$ is sufficiently close to $\sigma$. 
We suspect this might help to reduce the optimization error of AR-DAE, since the continuity of both $f_{ar}$ and $f_{ar}^*$ implies that
$f_{ar}(x,\sigma)$ only needs to refine $f_{ar}(x,\sigma')$ slightly if the latter already approximates the curvature of $f_{ar}^*(x,\sigma')$ well enough. 
Then by varying different $\sigma$ values, the conditional DAE is essentially interpolating between the gradient field of the log density function of interest and that of a Gaussian with the same expected value. 

\subsection{Approximation error}
\label{sec:approx_err_ardae}
\begin{figure}
\centering
\includegraphics[width=0.4\textwidth]{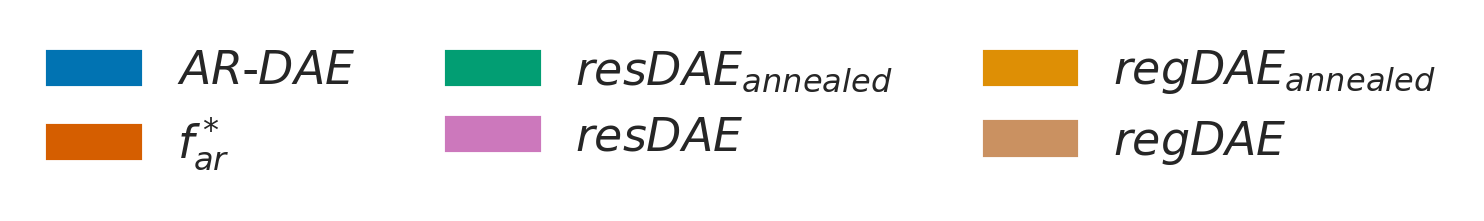}

\vspace*{-0.2cm}

\includegraphics[width=0.4\textwidth]{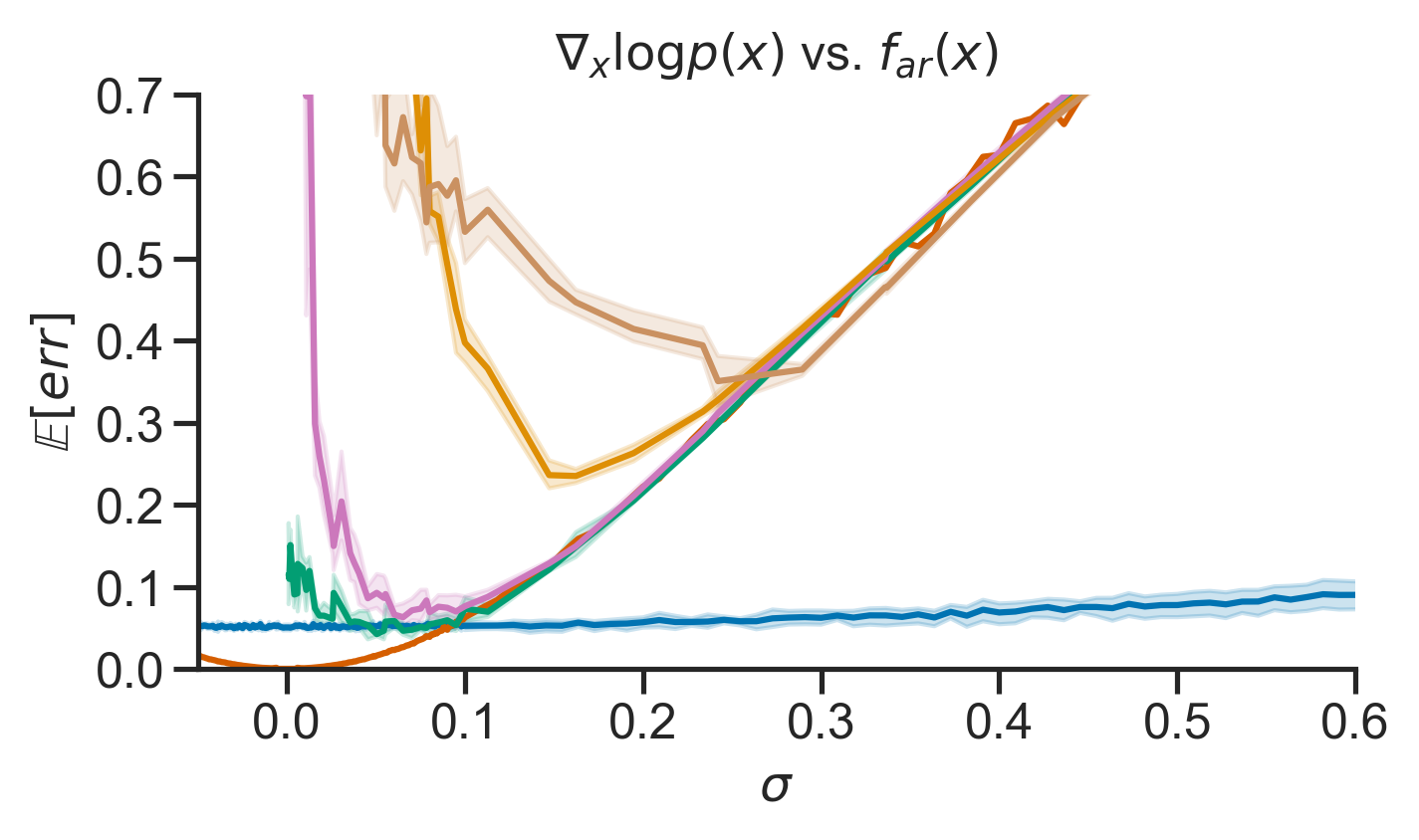}

\vspace*{-0.4cm}

\caption{
Approximating log density gradient of 1-D MoG.
\emph{AR-DAE}: the approximation error of the proposed method.
$f^*_{ar}$: the optimal DAE. %
\emph{resDAE}: a DAE of residual form (as well as loss rescaling).
\emph{regDAE}: a regular DAE. 
}
\label{fig:dae-1d-err}

\vspace*{-0.3cm}

\end{figure}

To study the approximation error with different variants of the proposed method, we consider a 1-dimensional mixture of Gaussians (MoG) with two equally weighted Gaussians centered at 2 and -2, and with a standard deviation of 0.5, as this simple distribution has a non-linear gradient function and an analytical form of the optimal gradient approximator $f^*$. 
See Appendix \ref{app:error_analysis} for the formula and an illustration of approximation with $f^*$ with different $\sigma$ values. 

We let $p$ be the density function of the MoG just described. 
For a given gradient approximator $f$, we estimate the expected error 
$\E_p[|\nabla_x \log p(x) - f|]$ using 1000 i.i.d. samples of $x\sim p$.
The results are presented in Figure \ref{fig:dae-1d-err}.
The curve of the expected error of the optimal $f_{ar}^*$ shows the asymptotic error indeed shrinks to $0$ as $\sigma\rightarrow0$, and it serves as a theoretical lower bound on the overall approximation error. 

Our ablation includes two steps of increments.
First, modifying the regular DAE (\emph{regDAE}) to be of the residual form (with loss rescaling, \emph{resDAE}) largely reduces the parameterization error and optimization error combined, as we use the same architecture for the reconstruction function of \emph{regDAE} and for the residual function of \emph{resDAE}. 
We also experiment with annealing the $\sigma$ values (as opposed to training each model individually): we take the model trained with a larger $\sigma$ to initialize the network that will be trained with a slightly smaller $\sigma$. 
Annealing significantly reduces the error and thus validates the continuity of the optimal $f_{ar}^*$. 
All four curves have a jump in the error when $\sigma$ gets sufficiently small, indicating the difficulty of optimization when the training signal diminishes.
This leads us to our second increment: amortization of training (i.e. AR-DAE). 
We see that not only does the error of AR-DAE decrease and transition more smoothly as $\sigma$ gets closer to $0$, but it also significantly outperforms the optimal $f_{ar}^*$ for large $\sigma$'s. 
We hypothesize this is due to the choice of the distribution over $\sigma$; $\cN(0,\delta^2)$ concentrates around $0$, which biases the training of $f_{ar}$ to focus more on smaller values of $\sigma$.

\section{Related Works}
Denoising autoencoders were originally introduced to learn useful representations for deep networks by \citet{VincentLBM08/icml, vincent2010stacked}. 
It was later on noticed by \citet{vincent2011connection} that the loss function of the residual form of DAE is equal to the expected quadratic error $||f-\nabla_x\log p_\sigma||^2$, where $p_\sigma(x') = \int p(x)\gN(x';x,\sigma^2 I)dx$ is the marginal distribution of the perturbed data, to which the author refers as \emph{denoising score matching}.
Minimizing expected quadratic error of this form is in general known as \emph{score matching} \citep{hyvarinen2005estimation}, where $\nabla_x \log p$ is referred to as the score \footnote{This is not to be confused with the score (or informant) in statistics, which is the gradient of the log likelihood function wrt the parameters.} of the density $p$. 
And it is clear now when we convolve the data distribution with a smaller amount of noise, the residual function $f$ tends to approximate $\nabla_x\log p(x)$ better. 
This is formalized by \citet{AlainB14/jmlr} as the limiting case of the optimal DAE. 
\citet{saremi2018deep, saremi2019neural} propose to use the residual and gradient parameterizations to train a deep energy model with denoising score matching. 

As a reformulation of score matching, instead of explicitly minimizing the expected square error of the score, the original work of \citet{hyvarinen2005estimation} proposes the \emph{Implicit score matching} and minimizes~
\begin{align}
\E_p\left[\frac{1}{2}||f(x)||^2 + \tr(\nabla_x f(x))\right]
.
\label{eq:ism}
\end{align}
\citet{SongGSE19} proposed a stochastic algorithm called the \emph{sliced score matching} to estimate the trace of the Jacobian, which reduces the computational cost from $\mathcal{O}(d_x^2)$ to $\mathcal{O}(d_x)$ (where $d_x$ is the dimensionality of $x$).
It was later noted by the same author that the computational cost of the sliced score matching is still much higher than that of the denoising score matching \citep{song2019generative}.

Most similar to our work are \citet{song2019generative} and \citet{bigdeli2020learning}.
\citet{song2019generative} propose to learn the score function of a data distribution, and propose to sample from the corresponding distribution of the learned score function using Langevin dynamics. 
They also propose a conditional DAE trained with a sequence of $\sigma$'s in decreasing order, and anneal the potential energy for the Langevin dynamics accordingly to tackle the mixing problem of the Markov chain. 
\citet{bigdeli2020learning} propose to match the score function of the data distribution and that of an implicit sampler.
As the resulting algorithm amounts to minimizing the reverse KL divergence, their proposal can be seen as a combination of \citet{song2019generative} and our work. 

Implicit density models are commonly seen in the context of likelihood-free inference \citep{MeschederNG17/icml, Tran2017hierarchical, li2017approximate, Huszar2019implicit}. 
Statistics of an implicit distribution are usually intractable, but there has been an increasing interest in approximately estimating the gradient of the statistics, such as the entropy \citep{LiT18/iclr,shi2018spectral} and the mutual information \citep{wen2020mutual}. 

\begin{figure}%
\centering
\begin{subfigure}[b]{0.09\textwidth}
    \parbox{0.1\textwidth}{\subcaption*{\textbf{1}}}%
    \hspace{0.5em}%
    \parbox{0.1\textwidth}{\includegraphics[width=10\linewidth]{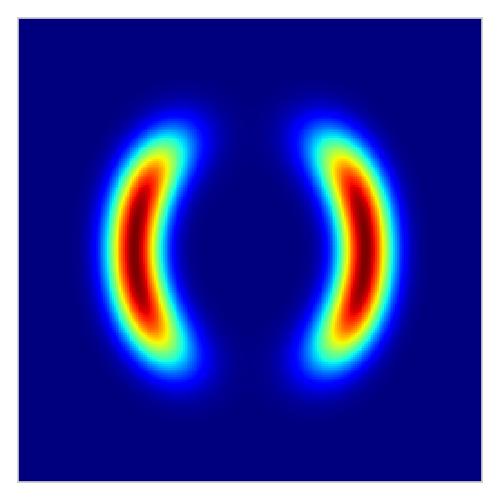}}
    
    \parbox{0.1\textwidth}{\subcaption*{\textbf{2}}}%
    \hspace{0.5em}%
    \parbox{0.1\textwidth}{
    \includegraphics[width=10\linewidth]{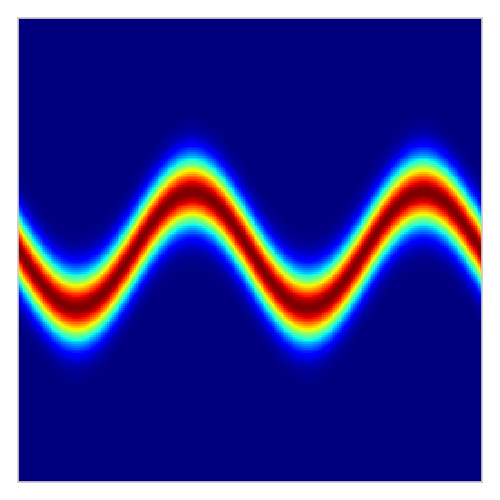}
    }
    
    \parbox{0.1\textwidth}{\subcaption*{\textbf{3}}}%
    \hspace{0.5em}%
    \parbox{0.1\textwidth}{
    \includegraphics[width=10\linewidth]{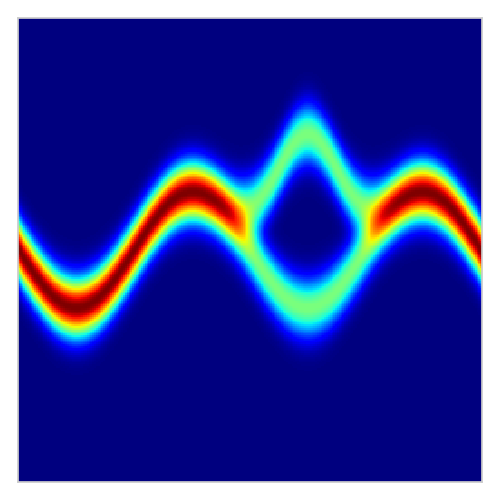}
    }
    
    \parbox{0.1\textwidth}{\subcaption*{\textbf{4}}}%
    \hspace{0.5em}%
    \parbox{0.1\textwidth}{
    \includegraphics[width=10\linewidth]{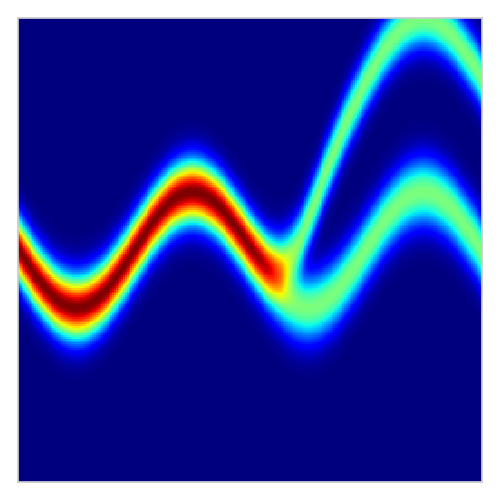}
    }

    \vspace*{-0.05cm}
    
    \caption*{\small $\quad\,\, \frac{1}{Z}e^{-U(x)}$}
    
    \vspace*{-0.25cm}
    
    \caption*{\scriptsize }
    
\end{subfigure}
\hspace{1em}%
\begin{subfigure}[b]{0.09\textwidth}
    \includegraphics[width=\linewidth]{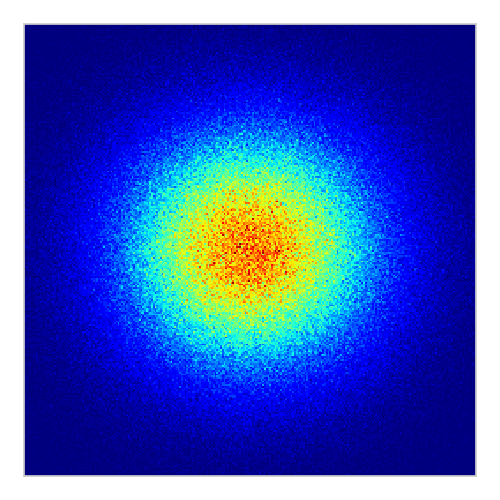}
    \includegraphics[width=\linewidth]{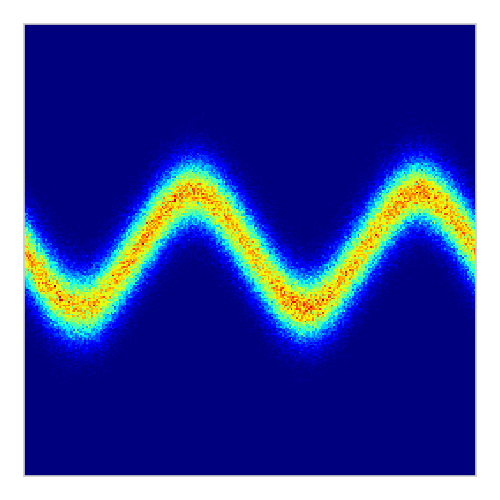}
    \includegraphics[width=\linewidth]{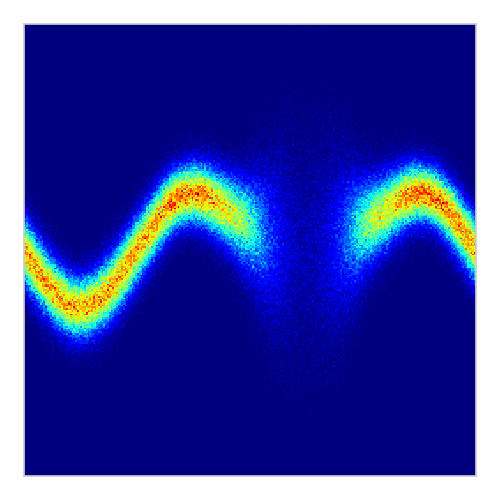}
    \includegraphics[width=\linewidth]{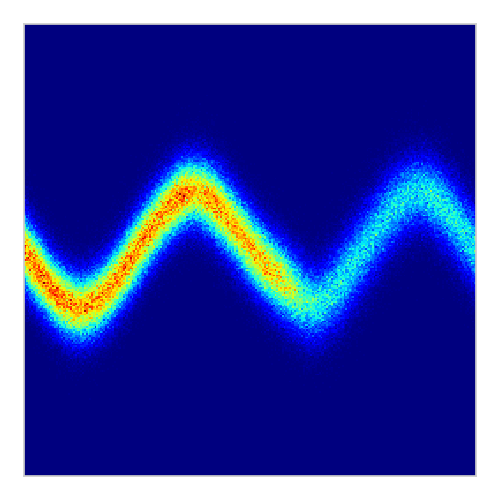}

    \vspace*{-0.05cm}
    
    \caption*{\small Aux}
    
    \vspace*{-0.2cm}
    
    \caption*{\scriptsize hierarchical}
    
\end{subfigure}
\begin{subfigure}[b]{0.09\textwidth}
    \includegraphics[width=\linewidth]{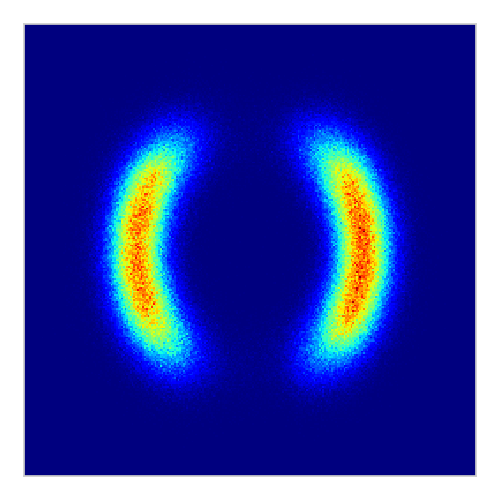}
    \includegraphics[width=\linewidth]{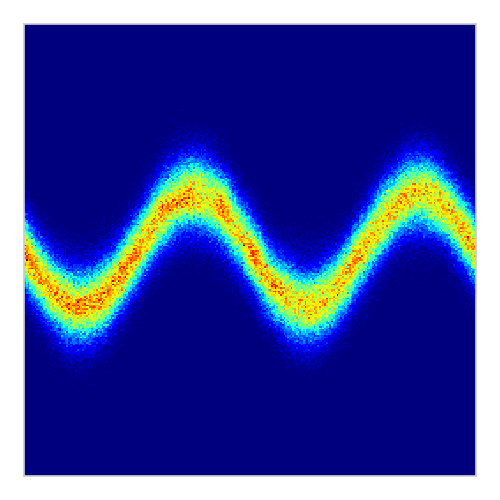}
    \includegraphics[width=\linewidth]{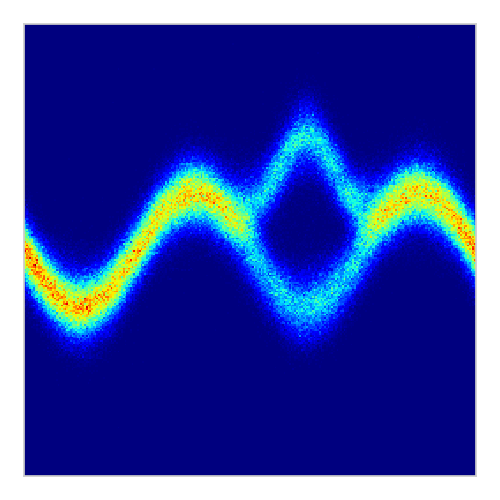}
    \includegraphics[width=\linewidth]{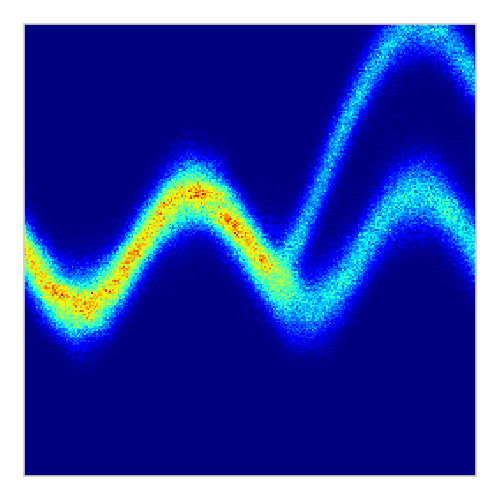}

    \vspace*{-0.05cm}
    
    \caption*{\small AR-DAE}
    
    \vspace*{-0.2cm}
    
    \caption*{\scriptsize hierarchical}
    
\end{subfigure}
\begin{subfigure}[b]{0.09\textwidth}
    \includegraphics[width=\linewidth]{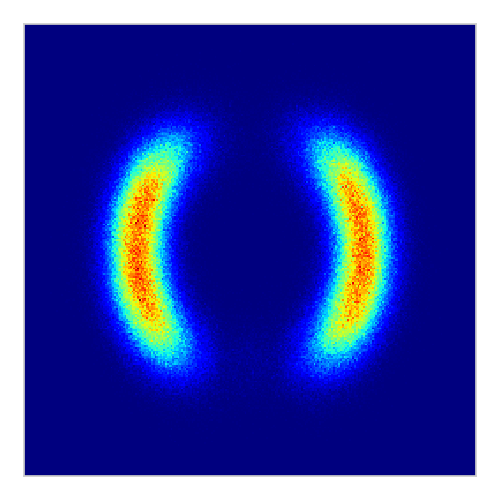}
    \includegraphics[width=\linewidth]{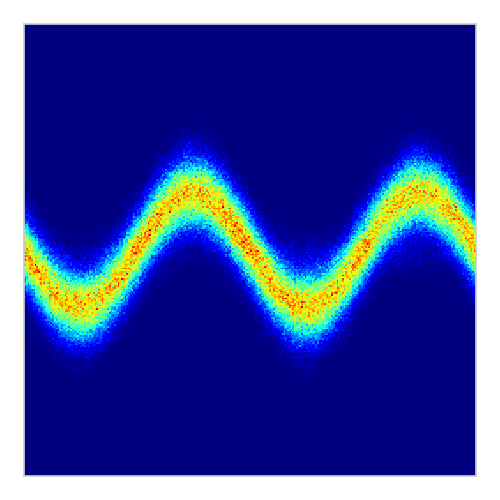}
    \includegraphics[width=\linewidth]{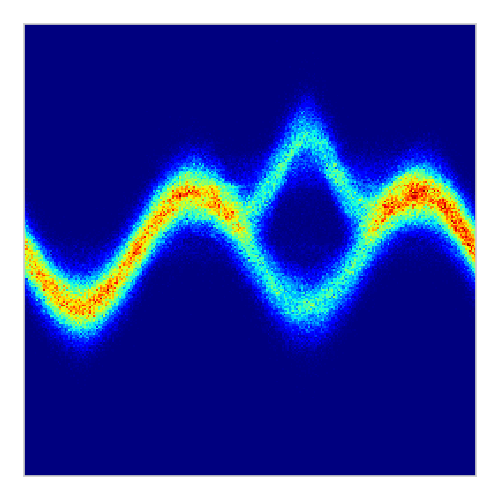}
    \includegraphics[width=\linewidth]{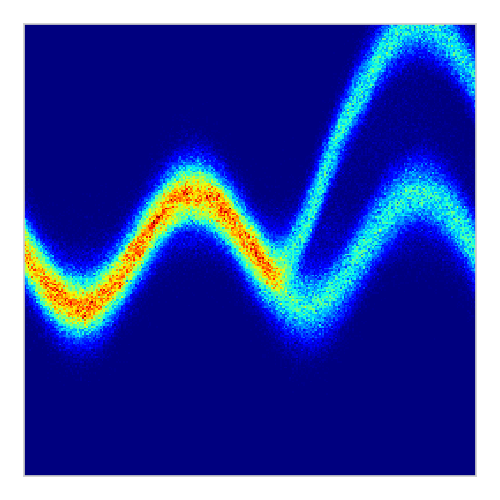}
    
    \vspace*{-0.05cm}

    \caption*{\small AR-DAE}
    
    \vspace*{-0.2cm}
    
    \caption*{\scriptsize implicit}
    
\end{subfigure}

\vspace*{-0.3cm}

\caption{
Fitting energy functions. 
First column: target energy functions. 
Second column: auxiliary variational method for hierarchical model.
Third column: hierarchical model trained with AR-DAE. 
Last column: implicit model trained with  AR-DAE. 
}
\label{fig:result-unnorm-density-estimation}

\vspace*{-0.5cm}

\end{figure}

\section{More Analyses and Experiments}

\begin{figure*}
\centering
\includegraphics[width=1.0\textwidth]{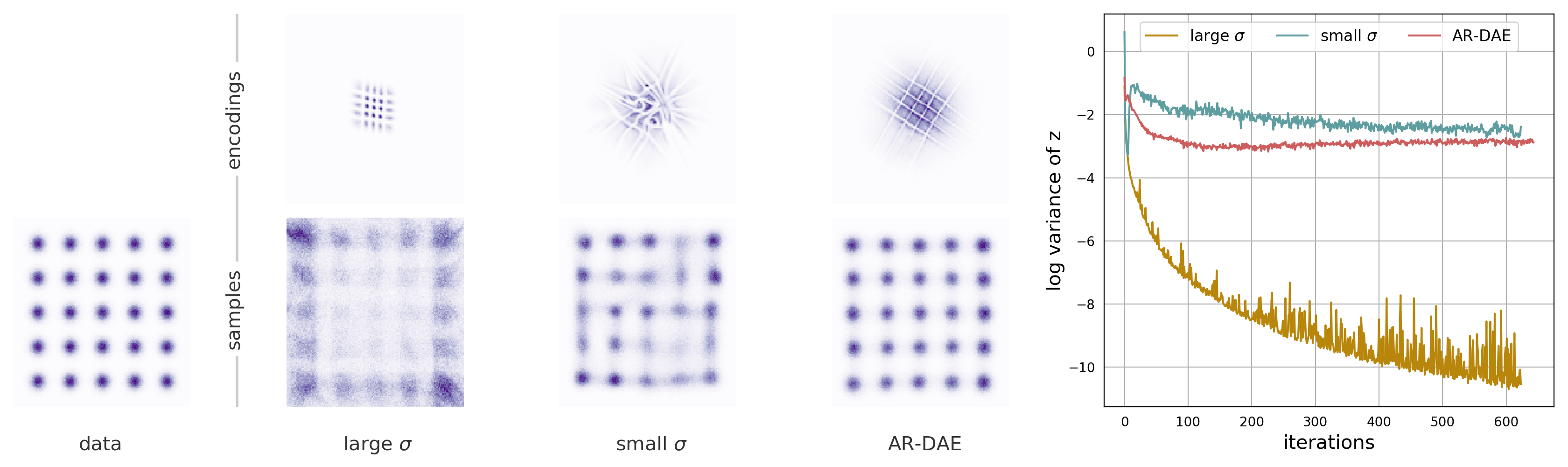}
\vspace{-0.5cm}
\caption{
Density estimation with VAE on $5 \times 5$ grid MoG. 1st column: data sampled from the MoG. 2nd column: 
VAE+DAE in residual form trained with a large $\sigma$ value. 
3rd column: VAE+DAE in residual form trained with a small $\sigma$ value. %
4th column: VAE+AR-DAE. %
Last column: averaged log variance of $z$ throughout training.
}
\label{fig:dae_mog_vae}

\vspace*{-0.3cm}

\end{figure*}

\subsection{Energy function fitting}
\label{sec:energy-function-fitting}
In Section \ref{sec:approx_err_ardae}, we have analyzed the error of approximating the gradient of the log-density function in the context of a fixed distribution. 
In reality, we usually optimize the distribution iteratively, and once the distribution is updated, the gradient approximator also needs to be updated accordingly to constantly provide accurate gradient signal. 
In this section, we use the proposed entropy gradient estimator to train an implicit sampler to match the density of some unnormalized energy functions. 

Concretely,
we would like to approximately sample from a target density which can be written as $p_{\tt{target}}(x) \propto \exp(-U(x))$, where $U(x)$ is an energy function. 
We train a neural sampler $g$ by minimizing the reverse Kullback-Leibler (KL) divergence
\eq{
\label{eq:reverse-kl}
&\kld(p_g(x) || p_{\tt{target}}(x)) \nn\\
&= - H(p_g(x)) + \eE_{x \sim p_g(x)} \lbs -\log p_{\tt{target}}(x) \rbs
.
}
where $p_g$ is the density induced by $g$. 
We use the target energy functions $U$ proposed in \citet{RezendeM15/icml} (see Table \ref{tab:target-energy-functions} for the formulas). 
The corresponding density functions are illustrated at the first column of Figure \ref{fig:result-unnorm-density-estimation}.

We consider two sampling procedures for $g$. 
The first one has a hierarchical structure: let $z$ be distributed by 
$\cN(0,I)$, and $x$ be sampled from the conditional $p_g(x|z):=\cN(\mu_g(z),\sigma_g^2(z))$ where $\mu_g$ and $\log\sigma_g$ are parameterized by neural networks. 
The resulting marginal density has the form $p_g(x) = \int p_g(x|z) p_g(z) dz$, which is computationally intractable due to the marginalization over $z$. 
We compare against the variational method proposed by \citet{AgakovB04a/iconip}, which lower-bounds the entropy by
\eq{
\label{eq:ent-lower-bound-aux}
H(p_g(x)) \ge - \eE_{x,z \sim p_g(x|z) p(z)} \lbs \log \f{ p_g(x|z) p(z) }{ h(z|x) } \rbs %
.%
}
Plugging (\ref{eq:ent-lower-bound-aux}) into (\ref{eq:reverse-kl}) gives us an upper bound on the KL divergence. 
We train $p_g$ and $h$ jointly by minimizing this upper bound \footnote{The normalizing constant of the target density will not affect the gradient $\nabla_x\log p_{\tt{target}}(x) =- \nabla_x U(x)$.} as a baseline.

The second sampling procedure has an implicit density: we first sample $z\sim \cN(0,I)$ and pass it through the generator $x=g(z)$.
We estimate the gradient of the negentropy of both the hierarchical and implicit models by following the approximate gradient of the log density $f_{ar}\approx\log p_g$.
The experimental details can be found in Appendix \ref{appendix:energy-fitting}.

As shown in Figure \ref{fig:result-unnorm-density-estimation}, the density learned by the auxiliary method sometimes fails to fully capture the target density. 
As in this experiment, we anneal the weighting of the cross-entropy term from $0.01$ to $1$, which is supposed to bias the sampler to be rich in noise during the early stage of training, the well-known mode seeking behavior of reverse KL-minimization should be largely mitigated.
This suggests the imperfection of the density trained with the auxiliary method is a result of the looseness of the variational lower bound on entropy, which leads to an inaccurate estimate of the gradient. 
On the other hand, the same hierarchical model and the implicit model trained with AR-DAE both exhibit much higher fidelity. 
This suggests our method can provide accurate gradient signal even when the sampler's distribution $p_g$ is being constantly updated. \footnote{We update $f_{ar}$ 5 times per update of $p_g$ to generate this figure; we also include the results with less updates of $f_{ar}$ in Appendix \ref{appendix:energy-fitting}.}

\afterpage{
\begin{table}[!t]%
\centering
\small
\begin{tabular}{lccc}
\toprule
&  \multicolumn{3}{c}{$\log p(x)$} \\
\midrule
 & \multicolumn{1}{c}{\it {MLP}} & \multicolumn{1}{c}{\it {Conv}} & \multicolumn{1}{c}{\it {ResConv}} \\
Gaussian$^\dagger$     &        -85.0          &          -81.9          & - \\
HVI aux$^\dagger$     &         -83.8          &          -81.6          & - \\
AVB$^\dagger$      &         -83.7          &         -81.7          & - \\
\midrule
Gaussian &          -84.40           &       -81.82          &       -80.75          \\
HVI aux &        -84.63          &        -81.87          & -80.80  \\
HVI AR-DAE (ours) &     \textbf{-83.42}     &      -81.46          &  -80.45 \\
IVI AR-DAE (ours) &     -83.62          &      \textbf{-81.26}     &     \textbf{-79.18}    \\
\bottomrule
\end{tabular}%
\caption{
Dynamically binarized MNIST. %
$^\dagger$Results taken from \citet{MeschederNG17/icml}. 
}
\label{tab:vae-result-dbmnist}
\end{table}

\begin{table}[!t]%
\centering
\small
\begin{tabular}{lc}
\toprule
    Model & $\log p(x)$ \\ \hline
{\small(Models with a trained prior)}\\
VLAE {\scriptsize\citep{chen2016variational}} & \textbf{-79.03} \\
PixelHVAE + VampPrior {\scriptsize\citep{tomczak2018vae}}
& -79.78 \\
\midrule
{\small(Models without a trained prior)}\\
VAE IAF {\scriptsize\citep{KingmaSJCCSW16/nips}} & -79.88 \\
VAE NAF {\scriptsize\citep{HuangKLC18/icml}} & -79.86 \\
Diagonal Gaussian & -81.43 \\
IVI AR-DAE (ours) & \textbf{-79.61} \\
\bottomrule
\end{tabular}
\caption{
Statically binarized MNIST.
}
\label{tab:vae-results-sbmnist}
\end{table}

\begin{figure}[!t]%
\centering
\includegraphics[width=\linewidth]{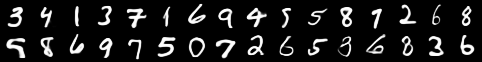}

\vspace*{-.2cm}

\caption{
Generated samples (the mean value of the decoder) of the IVI AR-DAE trained on statically binarized MNIST.
}
\label{fig:result-generated-samples-ivae}

\vspace*{-.2cm}
\end{figure}
}

\subsection{Variational autoencoder}%
\label{sec:vae}

In the previous section, we have demonstrated that AR-DAE can robustly approximate the gradient of the log density function that is constantly changing and getting closer to some target distribution. 
In this section, we move on to a more challenging application: likelihood-free inference for variational autoencoders (VAE, \citet{kingma2013auto, RezendeMW14/icml}). 
Let $p(z)$ be %
the standard normal.
We assume the data is generated by $x\sim p(x|z)$ which is parameterized by a deep neural network. 
To estimate the parameters, we maximize the marginal likelihood $\int p(x|z)p(z) dz$ of the data $x$, sampled from some data distribution $p_{\tt{data}}(x)$. 
Since the marginal likelihood is usually intractable, the standard approach is to maximize the \emph{evidence lower bound} (ELBO):
\eq{
\label{eq:elbo}
\log p(x) \ge \eE_{z \sim q(z|x)} \lbs \log p(x, z) - \log q(z|x) \rbs
,
}
where $q(z|x)$ is an amortized variational posterior distribution. 
The ELBO allows us to jointly optimize $p(x|z)$ and $q(z|x)$ with a unified objective.

Note that the equality holds 
\emph{iff}
$q(z|x)=p(z|x)$,
which
motivates using more flexible families of variational posterior. 
Note that this is more challenging for two reasons: the target distribution $p(z|x)$ is constantly changing and is conditional.
Similar to \citet{MeschederNG17/icml}, we parameterize a conditional sampler $z=g(\epsilon,x)$, $\epsilon\sim\mathcal{N}(0,I)$ with an implicit $q(z|x)$.
We use AR-DAE to approximate $\nabla_z \log q(z|x)$ and 
estimate the entropy gradient to update the encoder while maximizing the ELBO. %
To train AR-DAE, instead of fixing the prior variance $\delta$ in the $\cL_{ar}$ we adaptively choose $\delta$ for different data points. 
See Appendix \ref{appendix:vae} for a detailed description of the algorithm and heuristics we use.

\para{Toy dataset}
To demonstrate the difficulties in inference, %
we train a VAE with a 2-D latent space on a mixture of 25 Gaussians.
See Appendix \ref{app:subsec:vae-experiments} for the experimental details.

In Figure \ref{fig:dae_mog_vae}, we see that
if a fixed $\sigma$ is chosen to be too large
for the residual DAE, the DAE
tends to underestimate the gradient of the entropy,
so the variational posteriors collapse to point masses. %
If $\sigma$ is too small,
the DAE 
manages to maintain a non-degenerate variational posterior, %
but the inaccurate gradient approximation
results in a non-smooth encoder and 
poor generation quality. On the contrary, the same model trained with AR-DAE has a very smooth encoder that maps the data into a Gaussian-shaped, %
aggregated posterior %
and approximates the data distribution accurately.

\para{MNIST}
We first demonstrate the robustness of our method on different choices of architectures for VAE:
(1) a one-hidden-layer fully-connected network (denoted by {\it MLP}),
(2) a convolutional network (denoted by {\it Conv}), %
and (3) a larger convolutional network with residual connections (denoted by {\it ResConv}) from \cite{HuangKLC18/icml}.
The first two architectures are taken from \citet{MeschederNG17/icml} for a direct comparison with the adversarially trained implicit variational posteriors (AVB). 
We also implement a diagonal Gaussian baseline and the auxiliary hierarchical method (HVI aux, \citep{maaloe2016auxiliary}). 
We apply AR-DAE to estimate the entropy gradient of the hierarchical posterior and the implicit posterior (denoted by HVI AR-DAE and IVI AR-DAE, respectively).
As shown in Table \ref{tab:vae-result-dbmnist}, %
AR-DAE consistently improves the quality of inference in comparison to the auxiliary variational method and AVB, %
which is reflected by the better likelihood estimates.

We then compare our method with state-of-the-art VAEs evaluated on the statically binarized MNIST dataset \citep{larochelle2011neural}. %
We use the implicit distribution with the {\it ResConv} architecture following the previous ablation. 
As shown in Table \ref{tab:vae-results-sbmnist}, the VAE trained with AR-DAE demonstrates state-of-the-art performance among models with a fixed prior.
Generated samples are presented in Figure \ref{fig:result-generated-samples-ivae}.

\subsection{Entropy-regularized reinforcement learning}
\begin{figure*}[h]%
\centering

\vspace*{-0.3cm}

\begin{subfigure}[b]{0.38\textwidth}
\centering
\hspace*{0.4cm}\includegraphics[width=\textwidth]{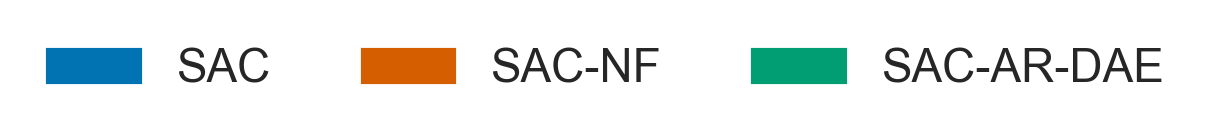}
\end{subfigure}

\vspace*{-0.2cm}

\begin{subfigure}[b]{\textwidth}
\centering
\hspace*{-0.3cm}\includegraphics[width=\textwidth]{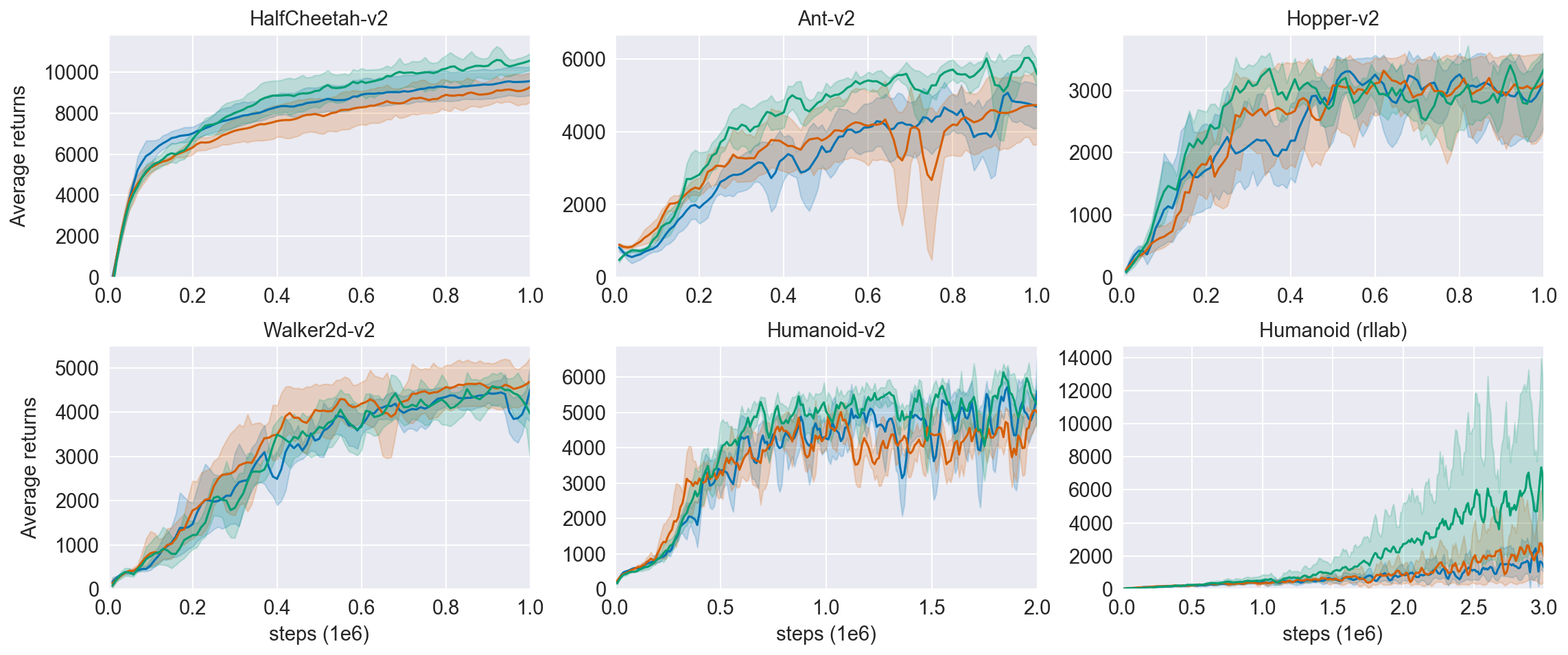}
\end{subfigure}

\vspace*{-0.3cm}

\caption{
Continuous control in reinforcement learning. SAC: soft actor-critic with diagonal Gaussian. SAC-NF: soft actor-critic with normalizing flows. SAC-AR-DAE: soft actor-critic with implicit distribution trained with AR-DAE. %
The shaded area indicates the standard error with 5 runs.
}
\label{fig:results-rl}

\vspace{-0.2cm}

\end{figure*}

We now apply AR-DAE to approximate entropy gradient in the context of reinforcement learning (RL). %
We use the
\emph{soft actor-critic} (SAC, \citet{HaarnojaZAL18/icml}), %
a state-of-the-art off-policy %
algorithm for continuous control that is designed to encourage exploration by regularizing the entropy of the policy. %
We train the policy $\pi(a|s)$ to minimize the following objective:
\eqst{
\cL(\pi) = \eE_{s \sim \cD} \lbs 
\kld \lbp
\pi(a|s) \middle\Vert \f{\exp \lbp Q(s, a) \rbp}{Z(s)}
\rbp
\rbs
,
}
where $\cD$ is a replay buffer of the past experience of the agent, %
$Q$ is a ``soft'' state-action value function that approximates the entropy-regularized expected return of the policy, and $Z(s) = \int_{a} \exp( Q(s, a) ) da$ is the normalizing constant of the Gibbs distribution.~A complete description of the SAC algorithm can be found in Appendix \ref{appendix:subsec:sac}. %
We compare with the original SAC that uses a diagonal Gaussian distribution as policy and a normalizing flow-based policy proposed by \citet{MazoureDDHP19}. %
We parameterize an implicit policy and use AR-DAE to approximate $\nabla_a \log \pi(a|s)$ to estimate~
\eq{
&\nabla_{\phi} \cL(\pi) \nn\\
&= \eE_{
    \substack{s \sim \cD \\
    a \sim \pi}} 
    \lbs \lbs\nabla_{a} \log \pi_{\phi} (a|s) -  \nabla_{a} Q(s, a) \rbs^{\intercal} \mathbf{J}_{\phi}g_{\phi}(\epsilon, s) \rbs
,
}
where $\pi(a|s)$ is implicitly induced by $a=g_{\phi}(\epsilon,s)$ with $\epsilon\sim\mathcal{N}(0,I)$.~We parameterize $f_{ar}$ as the gradient of a scalar function $F_{ar}$, so that $F_{ar}$ can be interpreted as the unnormalized log-density of the policy which will be used to update the soft Q-network. %
We run our experiments on fix continuous control environments from the 
OpenAI gym benchmark suite \cite{openaigym} and Rllab \cite{duan2016benchmarking}. %
The experimental details can be found in Appendix \ref{app:sac-ar-dae}.

The results are presented in Table \ref{tab:sac-result-max-main} and Figure \ref{fig:results-rl}. We see that SAC-AR-DAE using an implicit policy improves the performance over SAC-NF.~This also shows the approximate gradient signal of AR-DAE is stable and accurate even for reinforcement learning. The extended results for a full comparison of the methods are provided in Table \ref{tab:sac-result-max-full} and \ref{tab:sac-result-auc-full}.

\begin{table}[!t]
\centering
\small
\resizebox{0.48\textwidth}{!}{%
\begin{tabular}{lccc}
\toprule
&\multicolumn{1}{c}{SAC}&\multicolumn{1}{c}{SAC-NF}&\multicolumn{1}{c}{SAC-AR-DAE} \\
\midrule
HalfCheetah-v2 & 9695 $\pm$ 879 & 9325 $\pm$ 775 & \textbf{10907 $\pm$ 664} \\
Ant-v2 & 5345 $\pm$ 553 & 4861 $\pm$ 1091 & \textbf{6190 $\pm$ 128} \\
Hopper-v2 & \textbf{3563 $\pm$ 119} & 3521 $\pm$ 129 & 3556 $\pm$ 127 \\
Walker-v2 & 4612 $\pm$ 249 & 4760 $\pm$ 624 & \textbf{4793 $\pm$ 395} \\
Humanoid-v2 & 5965 $\pm$ 179 & 5467 $\pm$ 44 & \textbf{6275 $\pm$ 202} \\
Humanoid (rllab) & 6099 $\pm$ 8071 & 3442 $\pm$ 3736 & \textbf{10739 $\pm$ 10335} \\
\bottomrule
\end{tabular}%
}
\caption{Maximum average return. $\pm$ corresponds to one standard deviation over five random seeds.}
\label{tab:sac-result-max-main}
\end{table}

\begin{figure}
\centering
\includegraphics[width=0.23\textwidth]{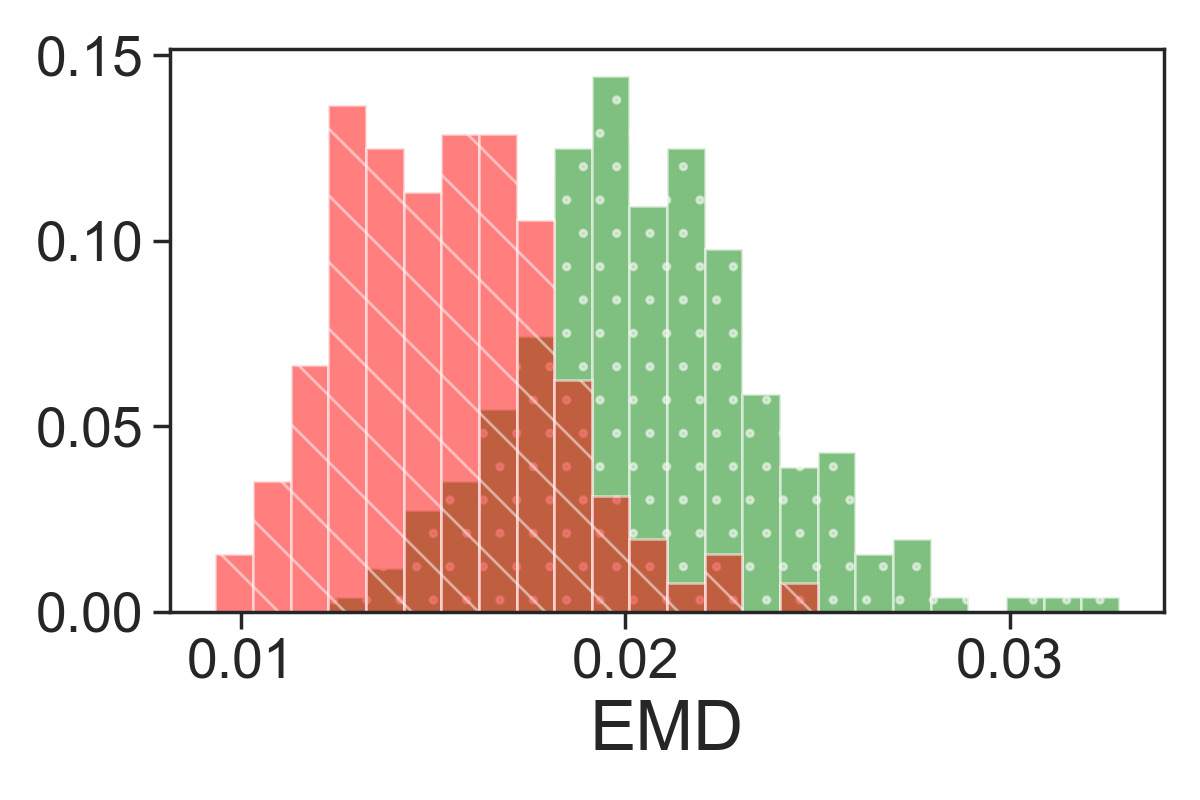}
\hfill
\includegraphics[width=0.23\textwidth]{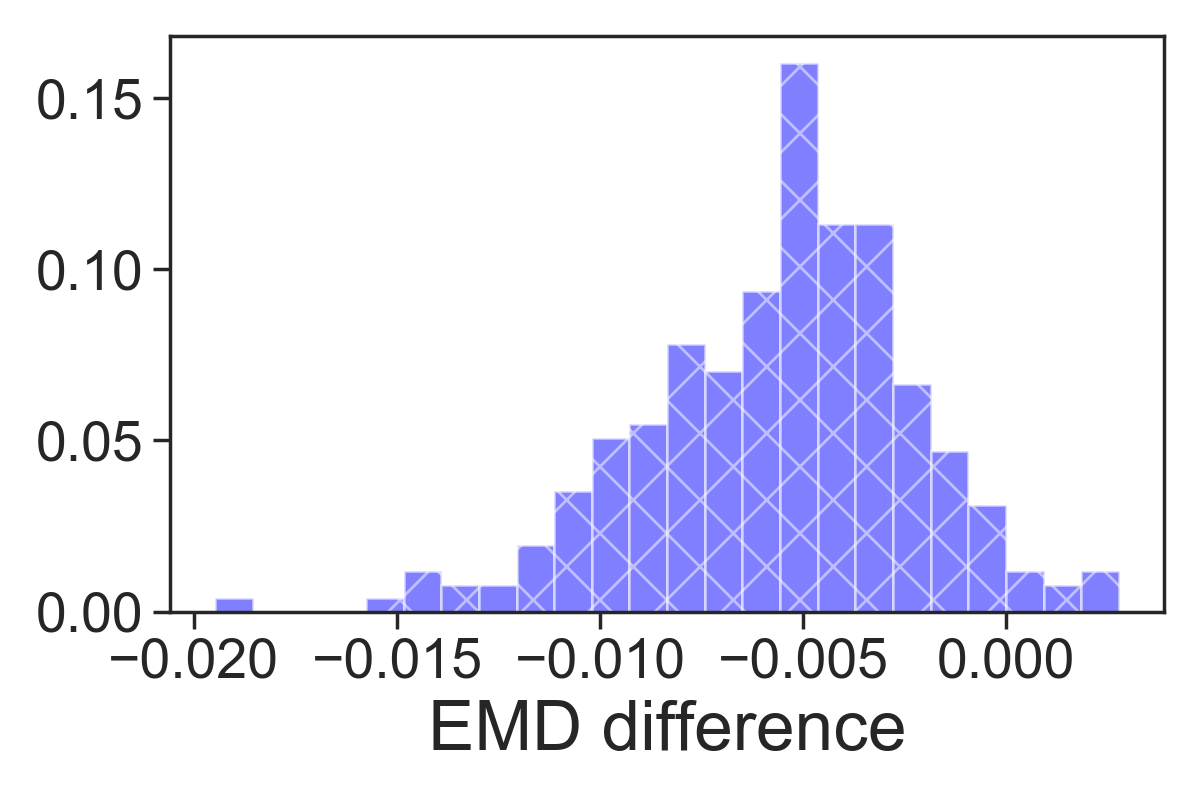}

\vspace*{-0.3cm}

\caption{Maximum entropy principle experiment. Left: estimated EMD of (red) the implicit distribution trained with AR-DAE and (green) the IAF. Right: the estimated EMD of the implicit distribution minus that of the IAF.}
\label{fig:max_ent}

\vspace*{-.4cm}

\end{figure}

\subsection{Maximum entropy modeling}
As a last application, we apply AR-DAE to solve the constrained optimization problem of the \emph{maximum entropy principle}.
Let $m\in\R^{10}$ be a random vector
and $B\in\R^{10\times10}$ be a random matrix
with $m_i$ and $B_{ij}$ drawn i.i.d. from $\cN(0,1)$. 
It is a standard result that among the class of real-valued random vectors $x\in\R^{10}$ satisfying the constraints $\E[x]=m$ and $\Var(x)=B^\top B$, $x\sim\cN(m,B^\top B)$ has the maximal entropy.
Similar to \citet{loaiza2017maximum}, we solve this constrained optimization problem but with an implicit distribution. 
We use the \emph{penalty method} and increasingly penalize the model to satisfy the constraints.
Concretely, let $\tilde{m}$ and $\tilde{C}$ be the sample mean and sample covariance matrix, respectively, estimated with a batch size of $128$.
We minimize the modified objective $-H(p_\theta(x)) + \lambda\sum_{j\in\{1,2\}} c_j^2 $,
where $c_1=||\tilde{m} - m||_2$ and $c_2=||\tilde{C}-B^\top B||_F$, with increasing weighting $\lambda$ on the penalty. 
We estimate the entropy gradient using AR-DAE, and compare against the inverse autoregressive flows (IAF, \citet{KingmaSJCCSW16/nips}).
At the end of training, we %
estimate the earth mover's distance (EMD) from $\cN(m,B^\top B)$. 

We repeat the experiment 256 times and report the histogram of EMD in Figure \ref{fig:max_ent}.
We see that most of the time the implicit model trained with AR-DAE has a smaller EMD, indicating the extra flexibility of arbitrary parameterization allows it to satisfy the geometry of the constraints more easily. %
We leave some more interesting applications suggested in \citet{loaiza2017maximum} for future work. 

\section{Conclusion}
We propose AR-DAE to estimate the entropy gradient of an arbitrarily parameterized data generator. 
We identify the difficulties in approximating the log density gradient with a DAE, and demonstrate the proposed method significantly reduces the approximation error. 
In theory, AR-DAE approximates the zero-noise limit of the optimal DAE, which is an unbiased estimator of the entropy gradient. 
We apply our method to a suite of tasks %
and empirically validate that AR-DAE provides accurate and reliable gradient signal to maximize entropy.

\section*{Acknowledgments}
We would like to thank Guillaume Alain for an insightful discussion on denoising autoencoders.~Special thanks to people who have provided their feedback and advice during discussion, including Bogdan Mazoure and Thang Doan for sharing the code on the RL experiment; to Joseph Paul Cohen for helping optimize the allocation of computational resources.~We thank CIFAR, NSERC and PROMPT for their support of this work. %

\bibliographystyle{icml2020}
\bibliography{refs}

\clearpage
\appendix
\onecolumn
\renewcommand{\thefigure}{S\arabic{figure}}
\setcounter{figure}{0}

\newpage
\section{Gradient of the entropy with respect to density functions}
\label{sec:appendix-entropy-gradient}
Consider a probability density function $p_g(x)$.~We assume $p_g$ is the pushforward of some prior distribution $p(z)$ by a mapping $g_{\theta}: z \mapsto x$. Our goal is to compute the gradient of the entropy of $p_g$ wrt the parameter $\theta$. 
Following \citet{roeder2017sticking}, we show that the entropy gradient can be rewritten as Equation (\ref{eq:entropy-gradient}).
\begin{proof}
By the law of the unconscious statistician$^*$ (LOTUS, Theorem 1.6.9 of \citet{durrett2019probability}), we have
\eqst{
\nabla_{\theta} H(p_{g}(x)) %
&= \nabla_\theta \eE_{x \sim p_{g}(x)} \left[ -\log p_{g}(x) \right] \nn\\
&\overset{*}{=} \nabla_\theta \eE_{z \sim p(z)} \left[ -\log p_{g}(g_{\theta}(z)) \right] \nn\\
&= -\nabla_\theta \int p(z) \log p_{g}(g_{\theta}(z)) dz \nn\\
&= \cancel{-\int p(z) \nabla_{\theta} \log p_{g}(x) |_{x=g_{\theta}(z)} dz}
    - \int p(z) [\nabla_x \log p_{g}(x)|_{x=g_{\theta}(z)}]^{\intercal} \mathbf{J}_{\theta}g_{\theta}(z) dz \nn\\
&= -\eE_{z \sim p(z)} \left[ [\nabla_x \log p_{g}(x)|_{x=g_{\theta}(z)}]^{\intercal} \mathbf{J}_{\theta}g_{\theta}(z) \right] %
.%
}
where the crossed-out term is due to the following identity
\eqst{
\eE_{z \sim p(z)} \left[ \nabla_{\theta} \log p_{g}(x) \middle\vert_{x=g_{\theta}(z)} \right]
&= \eE_{x \sim p_g(x)} \left[ \nabla_{\theta} \log p_{g}(x) \right] %
= \int p_g(x) \nabla_{\theta} \log p_{g}(x) dx \nn\\
&=\int \cancel{ p_g(x) } \f{1}{\cancel{p_g(x)}} \nabla_{\theta} p_{g}(x) dx %
= \nabla_{\theta} \int p_{g}(x) dx = \nabla_{\theta} 1 = 0 %
.
}
\end{proof}

\section{Properties of residual DAE}

\optimalres*
\begin{proof}
For simplicy, when the absolute value and power are both applied to a vector-valued variable, they are applied elementwise. 
The characterization of the optimal function $f^*$ can be derived by following \citet{AlainB14/jmlr}.
For the second part, the symmetry of the distribution of $u$ implies
\begin{align*}
f^*(x;\sigma)
&= \frac{-\eE_u[p(x-\sigma u)u]}{\sigma \eE_u[p(x-\sigma u)]}\\
&= \frac{\eE_u[p(x+\sigma u)u]}{\sigma \eE_u[p(x+\sigma u)]} = f^*(x; -\sigma),
\end{align*}
so we only need to show $f^*$ is continuous for $\sigma>0$. 
Since $p$ is bounded,
by the {dominated convergence theorem} (DOM), both 
$\eE_u[p(x-\sigma u)u]$ and $\eE_u[p(x-\sigma u)]$ are continuous for $\sigma>0$, and so is $f^*(x,\sigma)$.

Lastly, an application of L'Hôpital's rule gives
\begin{align*}
\lim_{\sigma\rightarrow0} f^*(x;\sigma) 
= \lim_{\sigma\rightarrow0} \frac{\frac{d}{d\sigma}\eE_u[p(x+\sigma u)u]}{\frac{d}{d\sigma} \sigma\eE_u[p(x+\sigma u)]},
\intertext{which by another application of DOM (since gradient of $p$ is bounded) is equal to}
\lim_{\sigma\rightarrow0} \frac{\eE[\nabla p(x+\sigma u)^\top uu]}{\eE[p(x+\sigma u)] + \sigma \eE[\nabla p(x+\sigma u)^\top u]}
.
\end{align*}
Applying DOM a final time gives 
$$\lim_{\sigma\rightarrow0}f^*(x;\sigma)  = \frac{\nabla p(x)\odot\E[u^2]}{p(x)} = \nabla \log p(x).$$
\end{proof}

\reslargesigma*
\begin{proof}
We rewrite the optimal gradient approximator as
$$f^*(x;\sigma)=\frac{1}{\sigma^2}\int \frac{\cN(u;0,I)p(x-\sigma u)}{\int \cN(u';0,I)p(x-\sigma u') du'} \cdot \sigma u\, du .$$
Changing the variables $\epsilon=\sigma u$ and $\epsilon'=\sigma u'$ gives
$$\frac{1}{\sigma^2}\int \frac{\cN(\epsilon/\sigma;0,I)p(x-\epsilon)}{\int \cN(\epsilon'/\sigma;0,I)p(x-\epsilon') d\epsilon'} \cdot \epsilon\, d\epsilon ,$$
which can be written as $\frac{1}{\sigma^2}\E_{q(\epsilon)}[\epsilon]$ where $q(\epsilon)\propto \cN(\epsilon/\sigma;0,I)p(x-\epsilon)$ is the change-of-variable density.

By DOM (applied to the numerator and denominator separately, since the standard Gaussian density is bounded), $\E_q[\epsilon]\rightarrow\int p(x-\epsilon) \epsilon\, d\epsilon$ 
as $\sigma\rightarrow\infty$. 
The latter integral is equal to $\E_{p}[X]-x$ (which can be seen by substituting $y=x-\epsilon$). 
\end{proof}

\section{Signal-to-noise ratio analysis on DAE's gradient}
\label{app:snr}
Fixing $x$ and $u$, the gradient of the L2 loss can be written as
\begin{align*}
\Delta := \nabla || u + \sigma f(x+\sigma u)||^2 
&= \nabla \left(\sum_{i}( u_i + \sigma f_i(x+\sigma u) )^2\right) 
= \sum_{i} \nabla  ( u_i + \sigma f_i(x+\sigma u) )^2
,
\end{align*}
where $i$ iterates over the entries of the vectors $u$ and $f$, and $\nabla$ denotes the gradient wrt the parameters of $f$. 
We further expand the gradient of the summand via chain rule, which yields
\begin{align*}
\nabla ( u_i + \sigma f_i(x+\sigma u))^2 
&= 2\sigma (u_i+\sigma f_i(x+\sigma u)) \nabla f_i(x+\sigma u) \\
&= 2\sigma \left(u_i\nabla 
\underbrace{f_i(x+\sigma u)}_{A} + \sigma
\underbrace{f_i(x+\sigma u)}_{B}\nabla
\underbrace{f_i(x+\sigma u)}_{C}\right)
.
\end{align*}

Taylor theorem with the mean-value form of the remainder allows us to approximate $f_i(x+\sigma u)$ by $f_i(x)$ as $\sigma$ is small:
\begin{align}
f_i(x+\sigma u) 
&= f_i(x) + \sigma \nabla_x f_i(\hat{x})^\top u \label{eq:zeroth_order_approx} \\
&= f_i(x) + \sigma \nabla_x f_i(x)^\top u + \frac{\sigma^2}{2} u^\top \grad_x^2 f_i(\tilde{x}) u
,
\label{eq:first_order_approx}
\end{align}
where $\nabla_x$ denotes the gradient wrt the input of $f$, and $\hat{x}$ and $\tilde{x}$ are points lying on the line interval connecting $x$ and $x+\sigma u$. 
Plugging (\ref{eq:first_order_approx}) into $A$ and (\ref{eq:zeroth_order_approx}) into $B$ and $C$ gives
\begin{align*}
&2\sigma\left(
u_i \nabla\left(f_i(x) + \sigma \nabla_x f_i(x)^\top u + \frac{\sigma^2}{2} u^\top \grad_x^2 f_i(\tilde{x}) u\right)
+ \sigma 
\left(f_i(x) + \sigma \nabla_x f_i(\hat{x})^\top u \right)
\nabla\left(f_i(x) + \sigma \nabla_x f_i(\hat{x})^\top u \right)
\right) \\
&=
2\sigma u_i \nabla f_i(x) + 
2\sigma^2 u_i \nabla \nabla_x f_i(x)^\top u + 
\sigma^3 u_i \nabla u^\top \grad_x^2 f_i(\tilde{x}) u \\
&\quad + 
2\sigma^2 f_i(x)\nabla f_i(x) +
2\sigma^3 f_i(x) \nabla\nabla_x f_i(\hat{x})^\top u +
2\sigma^3 \nabla_x \left(f_i(\hat{x})^\top u\right) \nabla f_i(x) +
2\sigma^4 \nabla_x \left(f_i(\hat{x})^\top u\right) \nabla\nabla_x f_i(\hat{x})^\top u
.
\end{align*}

With some regularity conditions (DOM-style assumptions), 
marginalizing out $u$ and taking $\sigma$ to be arbitrarily small yield
\begin{align*}
\E_u[ \Delta ]
&= \sum_{i} 2\sigma^2 \nabla \frac{\partial}{\partial x_i} f_i(x) + 2\sigma^2 f_i(x)\nabla f_i(x) + o(\sigma^2) \\
&= 2\sigma^2 \nabla\left(\tr(\nabla_x f(x)) + \frac{1}{2} ||f(x)||^2\right) + o(\sigma^2)
.
\end{align*}
In fact, we note that the first term is the stochastic gradient of the implicit score matching objective (Theorem 1, \citet{hyvarinen2005estimation}), but it vanishes at a rate $\mathcal{O}(\sigma^2)$ as $\sigma^2\rightarrow0$.

For the second moment, 
similarly,
$$\E_u[\Delta \Delta^\top] = 
4\sigma^2 \sum_i \nabla f_i(x) \nabla f_i(x)^\top + o(\sigma^2).$$
As a result,
$$\frac{\E[\Delta]}{\sqrt{\Var(\Delta)}}
=\frac{\E[\Delta]}{\sqrt{
\E(\Delta\Delta^\top) - \E(\Delta)\E(\Delta)^\top 
}}
=\frac{\mathcal{O}(\sigma^2)}{\sqrt{\mathcal{O}(\sigma^2) - \mathcal{O}(\sigma^4)}}
= \mathcal{O}(\sigma)
.
$$

\section{Experiment: Error analysis}
\subsection{Main experiments}
\label{app:error_analysis}

\begin{figure}[h]
\centering
\includegraphics[width=0.75\textwidth]{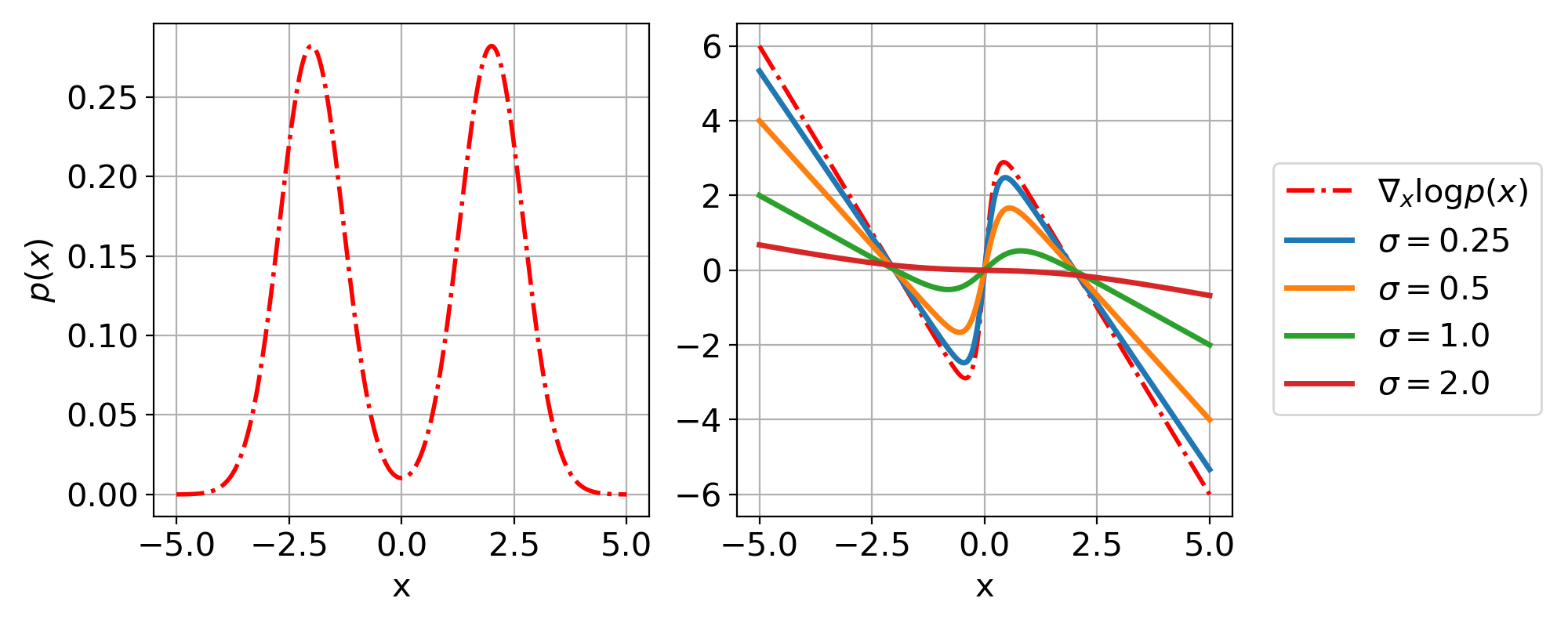}

\vspace*{-0.5cm}

\caption{Left: Density function of mixture of Gaussians. 
Right: gradient of the log density function (dotdash line) and gradient approximations using optimal DAE with different $\sigma$ values (solid lines). }
\label{fig:1d-mog-data-exapmles}
\end{figure}

\para{Dataset and optimal gradient approximator}
As we have described in Section \ref{sec:approx_err_ardae}, we use the mixture of two Gaussians to analyze the approximation error (see Figure \ref{fig:1d-mog-data-exapmles} (left)). Formally, we define $p(x) = 0.5 \cN(x; 2, 0.25) + 0.5 \cN(x; -2, 0.25)$. 
For notational convenience, we let $p_1$ and $p_2$ be the density functions of these two Gaussians, respectively. %
We obtain $\nabla_x \log p(x) $ by differentiating $\log p(x)$ wrt $x$ using auto-differentiation library such as PyTorch \citep{paszke2017automatic}. 
With some elementary calculation, we can expand the formula of the optimal gradient approximator $f^*$ as,
\eqst{
f^*(x;\sigma)
= \frac{-\eE_u[p(x-\sigma u)u]}{\sigma \eE_u[p(x-\sigma u)]}
= \frac{-\sum_{i=1}^{2} S'_i \mu'_i}{ \sigma \sum_{i=1}^{2} S'_i }
,
}
where $S'_i = \nicefrac{1}{\sqrt{2 \pi (0.5^2 + 1^2)}} \exp \lbp -\nicefrac{(\mu_i+x/\sigma)^2}{2(0.5^2 + 1^2)} \rbp$ for $i \in {1, 2}$, $\mu_1 = -2$, and $\mu_1 = 2$.
\begin{proof} The numerator $\eE_u[\eE_u[p(x-\sigma u)u](x-\sigma u)u]$ can be rewritten as follows:
\eqst{
\eE_u[p(x-\sigma u)u]
&= \int \lbp 0.5 p_1(x-\sigma u) + 0.5 p_2(x-\sigma u) \rbp p(u) u \,du %
= \f{0.5}{\sigma} \sum_{i=1}^{2} S'_i \int \cN(u; \mu'_i, \sigma'_i) u \,du %
= \f{0.5}{\sigma} \sum_{i=1}^{2} S'_i \mu'_i %
,
}
where $S'_i = \nicefrac{1}{\sqrt{2 \pi (0.5^2 + 1^2)}} \exp \lbp -\nicefrac{(\mu_i+x/\sigma)^2}{2(0.5^2 + 1^2)} \rbp$ for $i \in {1, 2}$, $\mu_1 = -2$, and $\mu_1 = 2$.

The second equality comes from the fact that all $p_1$, $p_2$, and $p(u)$ are normal distributions, and thus we have
\eqst{
p_i(x-\sigma u)p(u)
= \f{1}{\sigma} p_i \lbp u - x/\sigma \rbp p(u) %
= \f{1}{\sigma} S'_i \cN(u; \mu'_i, \sigma'_i) %
.
}
Similarly, we can rewrite the denominator as $\eE_u[p(x-\sigma u)] = \f{0.5}{\sigma} \sum_{i=1}^{2} S'_i$.
\end{proof}

\para{Experiments}
For AR-DAE, we indirectly parameterize it as the gradient of some scalar-function (which can be thought of as an unnormalized log-density function); \emph{i.e.} we define a scalar function and use its gradient wrt the input vector. The same trick has also been employed in recent work by \citet{saremi2018deep, saremi2019neural}. We use the network architecture with the following configuration\footnote{$[d_{\tt{input}}, d_{\tt{output}}]$ denotes a fully-connected layer whose input and output feature sizes are $d_{\tt{input}}$ and $d_{\tt{output}}$, respectively.}: $[2+1, 256]$ + $[256, 256] \times 2$ + $[256, 1]$, with the {\tt Softplus} activation function. We use the same network architecture for {\it resDAE} except it doesn't condition on $\sigma$. For {\it regDAE}, the network is set to reconstruct input.

All models are trained for 10k iterations with a minibatch size of 256.
We use the Adam optimizer for both AR-DAE and the generator, with the default $\beta_1=0.9$ and $\beta_2=0.999$.
For all models, the learning rate is initially set to 0.001 and is reduced by half every 1k iterations during training.

For {\it regDAE} and {\it resDAE}, we train models individually for every $\sigma$ value in Figure \ref{fig:dae-1d-err}.
For {\it regDAE\textsubscript{\it annealed}} and {\it resDAE\textsubscript{\it annealed}}, we anneal $\sigma$ from 1 to the target value.%
For AR-DAE, $\delta$ is set to 0.05 and we sample 10 $\sigma$'s from $N(0, \delta^2)$ for each iteration.
We train all models five times and present the mean and its standard error in the figures.

\subsection{Symmetrizing the distribution of $\sigma$}
In Section \ref{subsec:ardae}, we argue that neural networks are not suitable for extrapolation (vs. interpolation), to motivate the use of a symmetric prior over $\sigma$. 
To contrast the difference, %
we sample $\sigma \sim N(0, \delta^2)$ and compare two different types of $\sigma$-conditioning: (1) conditioning on $\sigma$, and (2) conditioning on $|\sigma|$. We use the same experiment settings in the previous section, %
but we use a hypernetwork \citep{HaDL17/iclr} that takes $\sigma$ (resp. $|\sigma|$) as input and outputs the parameters of AR-DAE, to force AR-DAE to be more dependent on the value of $\sigma$ (resp. $|\sigma|$). The results are shown in Figure \ref{fig:dae-1d-err-intra-vs-extra}.

We see that the two conditioning methods result in two distinct approximation behaviors. First, when AR-DAE only observes positive values, it fails to extrapolate to the $\sigma$ values close to 0. When a symmetric $\sigma$ distribution is used, the approximation error of AR-DAE is more smooth. Second, we notice that the symmetric $\sigma$ distribution bias $f_{ar}$ to focus more on small $\sigma$ values. %
Finally, the asymmetric distribution helps AR-DAE reduce the approximation error for some $\sigma$. We speculate that AR-DAE with the asymmetric $\sigma$ distribution has two times higher to observe small $\sigma$-values during training, and thus improves the approximation. In general, we observe that the stability of the approximation is important for our applications, in which case AR-DAE need to adapt constantly in the face of non-stationary distributions. %

\begin{figure}[h]
    \centering
    \includegraphics[width=0.45\textwidth]{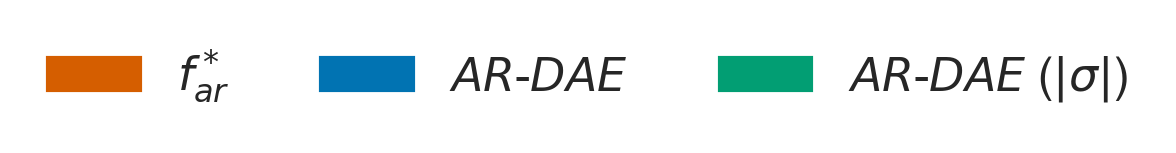}
    
    \vspace*{-0.2cm}
    
    \includegraphics[width=0.5\textwidth]{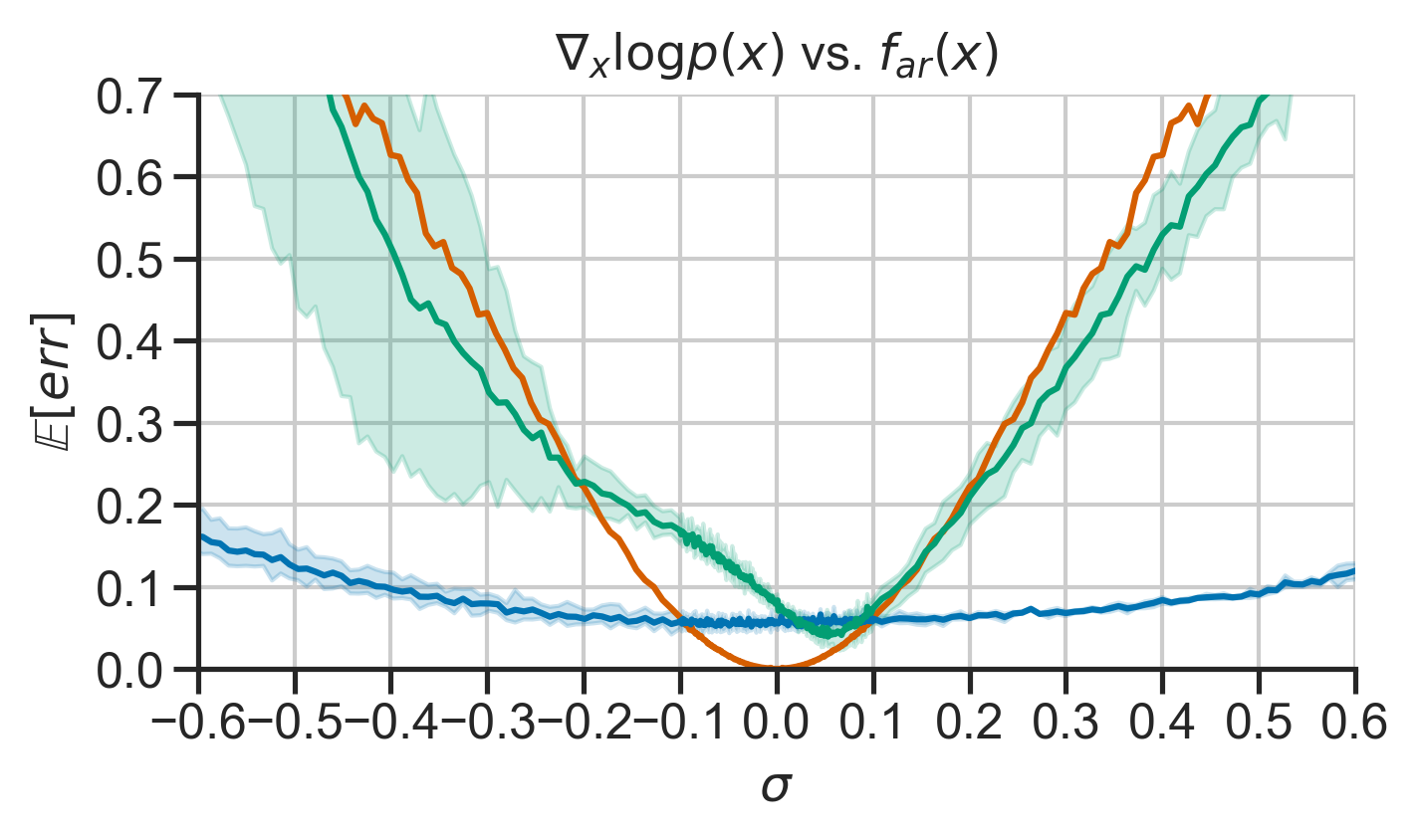}
    
    \vspace*{-0.4cm}
    
    \caption{Comparison of two $\sigma$-conditioning methods to approximate log density gradient of 1D-MOG. AR-DAE: conditioning on $\sigma$. AR-DAE ($|\sigma|$): conditioning on $|\sigma|$. $\sigma$ is sampled from $N(0, \delta)$ for all experiments.
    }
    \label{fig:dae-1d-err-intra-vs-extra}
\end{figure}

\section{Experiment: Energy Fitting}
\label{appendix:energy-fitting}

\begin{table*}[!h]
\centering
\begin{tabular}{l}
    \toprule
    \textbf{Potential} $U(\bz)$ \\ 
    \midrule
    \textbf{1}: $\f{1}{2} \lbp \f{\lbV \bz \rbV - 2}{0.4}\rbp^2
    - \ln \lbp
    e^{-\f{1}{2} \lbs \f{z_1-2}{0.6} \rbs^2}
    + e^{-\f{1}{2} \lbs \f{z_1+2}{0.6} \rbs^2}
    \rbp$ \\
    \textbf{2}: $\f{1}{2} \lbp \f{z_2 - w_1(\bz)}{0.4}\rbp^2$ \\
    \textbf{3}: $
    - \ln \lbp
    e^{-\f{1}{2} \lbs \f{z_2 - w_1(\bz)}{0.35} \rbs^2}
    + e^{-\f{1}{2} \lbs \f{z_2 - w_1(\bz) + w_2(\bz)}{0.35} \rbs^2}
    \rbp$ \\
    \textbf{4}: $
    - \ln \lbp
    e^{-\f{1}{2} \lbs \f{z_2 - w_1(\bz)}{0.4} \rbs^2}
    + e^{-\f{1}{2} \lbs \f{z_2 - w_1(\bz) + w_3(\bz)}{0.35} \rbs^2}
    \rbp$ \\
    \midrule
    where $w_1(\bz) = \sin \lbp \f{2 \pi z_1}{4} \rbp$,
    $w_2(\bz) = 3  e^{-\f{1}{2} \lbs \f{z_1 - 1}{0.6} \rbs^2}$,
    $w_3(\bz) = 3  \sigma \lbp \f{z_1 - 1}{0.3}\rbp$, 
    $\sigma(x) = \f{1}{1+e^{-x}}$.
    \\
    \bottomrule
\end{tabular}
\caption{
The target energy functions introduced in \citet{RezendeM15/icml}.
}
\label{tab:target-energy-functions}
\end{table*}

\subsection{Main experiments}
Parametric densities trained by minimizing the reverse KL divergence tend to avoid ``false positive'', a well known problem known as the zero-forcing property \citep{minka2005divergence}. %
To deal with this issue, we minimize a modified objective:
\eq{
\label{eq:reverse-kl-alpha}
{\kld}_{\alpha}(p_g(x) || p_{\tt{target}}(x)) = - H(p_g(x)) - \alpha \eE_{x \sim p_g(x)} \lbs \log p_{\tt{target}}(x) \rbs
,
}
where $\alpha$ is annealed from a small value to 1.0 throughout training.
This slight modification of the objective function ``convexifies'' the loss landscape and makes it easier for the parametric densities to search for the lower energy regions. For AR-DAE training, we use Equation (\ref{eq:ardae}) with a fixed prior variance $\delta = 0.1$.

For all experiments, we use a three-hidden-layer MLP for both hierarchical distribution as well as implicit distribution. More specifically, the generator network for the hierarchical distribution has the following configuration: $[d_z, 256]$ + $[256, 256] \times 2$ + $[256, 2] \times 2$. $d_z$ indicates the dimension of the prior distribution $p(z)$ and is set to 2. The last two layers are for mean and log-variance\footnote{diagonal elements of the covariance matrix in log-scale} of the conditional distribution $p_g(x|z)$. For the auxiliary variational method, the same network architecture is used for $h(z|x)$ in Equation (\ref{eq:ent-lower-bound-aux}). When we train the hierarchical distribution with AR-DAE, we additionally clamp the log-variance to be higher than -4. Similar to the hierarchical distribution, the generator of the implicit distribution is defined as, $[d_z, 256]$ + $[256, 256] \times 2$ + $[256, 2]$. Unlike the hierarchical distribution, $d_z$ is set to 10. {\tt ReLU} activation function is used for all but the final output layer.

For AR-DAE, we directly parameterize the residual function $f_{ar}$. We use the following network architecture: $[2, 256]$ + $[256, 256] \times 2$ + $[256, 2]$. {\tt Softplus} activation function is used.

Each model is trained for 100,000 iterations with a minibatch size of 1024. We update AR-DAE $N_d$ times per generator update. For the main results, we set $N_d = 5$.
We use the Adam optimizer for both the generator and AR-DAE, where $\beta_1=0.5$ and $\beta_2=0.999$.
The learning rate for the generator is initially set to 0.001 and is reduced by 0.5 for every 5000 iterations during training. AR-DAE's learning rate is set to 0.001.
To generate the figure, we draw 1M samples from each model to fill up 256 equal-width bins of the 2D histogram.

\subsection{Effect of the number of updates ($N_d$) of the gradient approximator}
In addition to the main results, we also analyze how the number of updates of AR-DAE per generator update affects the quality of the generator. We use the same implicit generator and AR-DAE described in the main paper, but vary $N_d$ from 1 to 5. The result is illustrated in Figure \ref{fig:result-unnorm-density-estimation-num-updates}. In principle, the more often we update AR-DAE, the more accurate (or up-to-date) the gradient approximation will be. This is corroborated by the improved quality of the trained generator.

\begin{figure}[!h]%
\centering
\begin{subfigure}[b]{0.0896\textwidth}
\centering
\begin{subfigure}[b]{\textwidth} %
    \parbox{0.1\textwidth}{\subcaption*{\textbf{1}}}%
    \hspace{0.5em}%
    \parbox{0.1\textwidth}{\includegraphics[width=10\linewidth]{figs/fit/fit-po1-data.png}}
    
    \parbox{0.1\textwidth}{\subcaption*{\textbf{2}}}%
    \hspace{0.5em}%
    \parbox{0.1\textwidth}{
    \includegraphics[width=10\linewidth]{figs/fit/fit-po2-data.png}
    }
    
    \parbox{0.1\textwidth}{\subcaption*{\textbf{3}}}%
    \hspace{0.5em}%
    \parbox{0.1\textwidth}{
    \includegraphics[width=10\linewidth]{figs/fit/fit-po3-data.png}
    }
    
    \parbox{0.1\textwidth}{\subcaption*{\textbf{4}}}%
    \hspace{0.5em}%
    \parbox{0.1\textwidth}{
    \includegraphics[width=10\linewidth]{figs/fit/fit-po4-data.png}
    }

    \vspace*{-0.05cm}
    
    \caption*{\small $\quad\,\, \frac{1}{Z}e^{-U(x)}$}
    
    \vspace*{-0.25cm}
    
    \caption*{\scriptsize }
    
\end{subfigure}

\vspace*{-0.35cm}

\caption{}

\end{subfigure}
\hspace{3em}
\begin{subfigure}[b]{0.28\textwidth}
\centering
\begin{subfigure}[b]{0.32\textwidth} %
    \includegraphics[width=\linewidth]{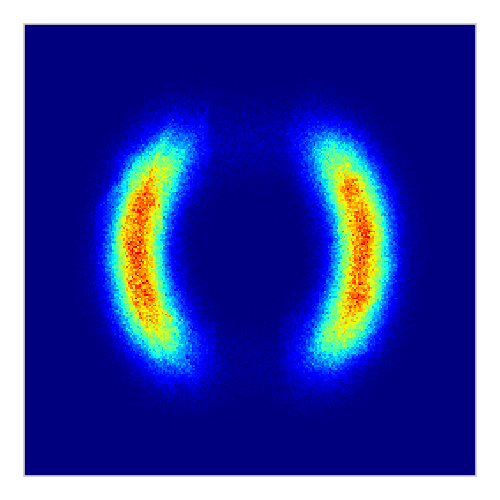}
    \includegraphics[width=\linewidth]{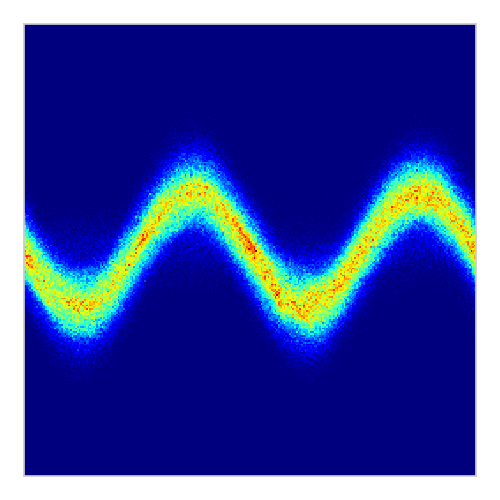}
    \includegraphics[width=\linewidth]{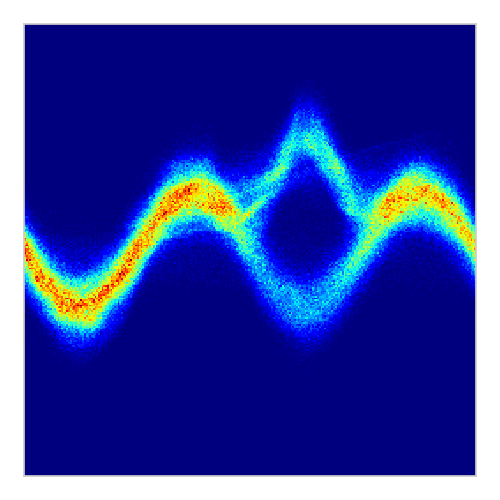}
    \includegraphics[width=\linewidth]{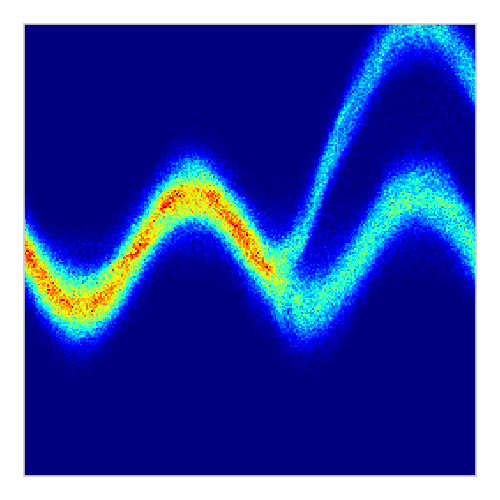}
    
    \vspace*{-0.05cm}

    \caption*{\small $N_d=1$}
    
    \vspace*{-0.2cm}
    
    \caption*{\scriptsize}
    
\end{subfigure}
\begin{subfigure}[b]{0.32\textwidth} %
    \includegraphics[width=\linewidth]{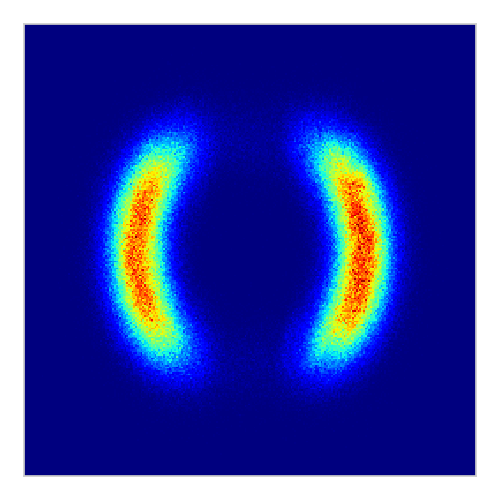}
    \includegraphics[width=\linewidth]{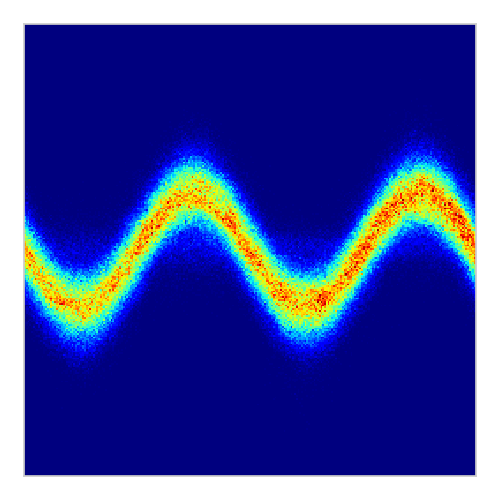}
    \includegraphics[width=\linewidth]{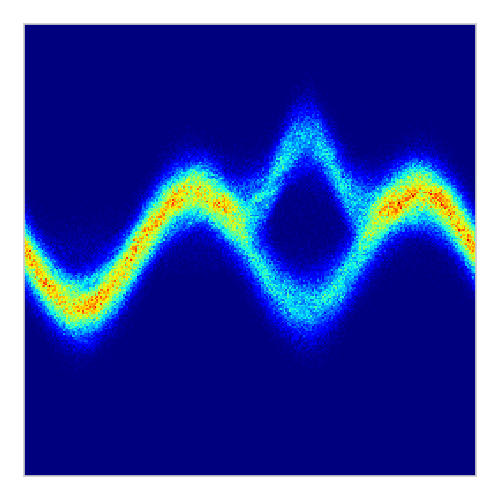}
    \includegraphics[width=\linewidth]{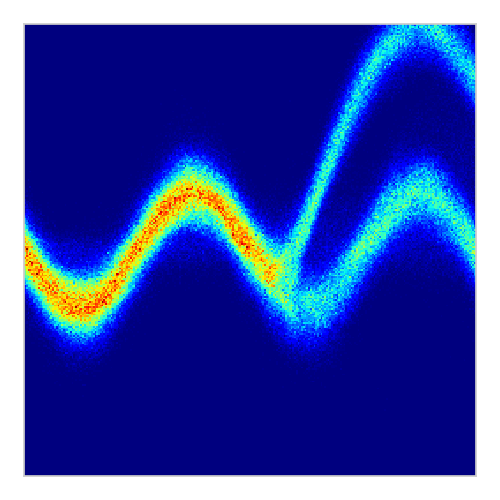}
    
    \vspace*{-0.05cm}

    \caption*{\small  $N_d=2$}
    
    \vspace*{-0.2cm}
    
    \caption*{\scriptsize}
    
\end{subfigure}
\begin{subfigure}[b]{0.32\textwidth} %
    \includegraphics[width=\linewidth]{figs/fit/po1-ardae-im-nd5-nstd10-mnh3.png}
    \includegraphics[width=\linewidth]{figs/fit/po2-ardae-im-nd5-nstd10-mnh3.png}
    \includegraphics[width=\linewidth]{figs/fit/po3-ardae-im-nd5-nstd10-mnh3.png}
    \includegraphics[width=\linewidth]{figs/fit/po4-ardae-im-nd5-nstd10-mnh3.png}
    
    \vspace*{-0.05cm}

    \caption*{\small $N_d=5$}
    
    \vspace*{-0.2cm}
    
    \caption*{\scriptsize}
    
\end{subfigure}

\vspace*{-0.35cm}

\caption{}

\end{subfigure}
\hspace{2em}
\begin{subfigure}[b]{0.28\textwidth}
\centering
\begin{subfigure}[b]{0.32\textwidth} %
    \includegraphics[width=\linewidth]{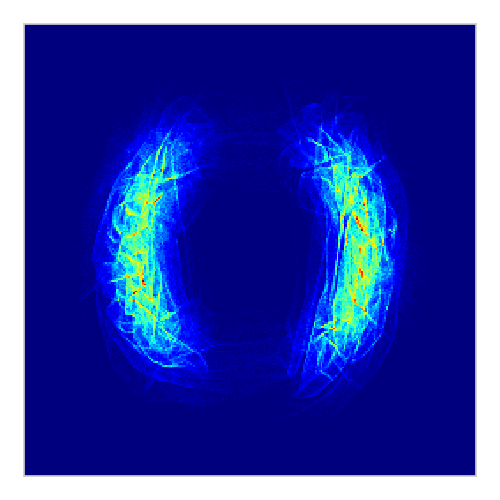}
    \includegraphics[width=\linewidth]{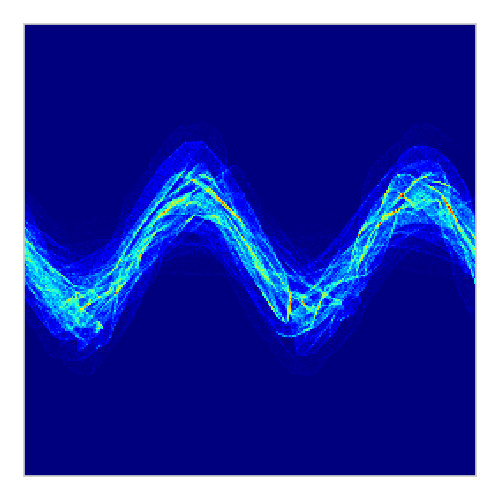}
    \includegraphics[width=\linewidth]{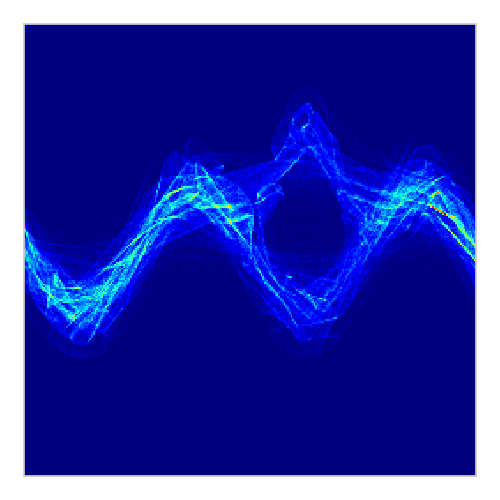}
    \includegraphics[width=\linewidth]{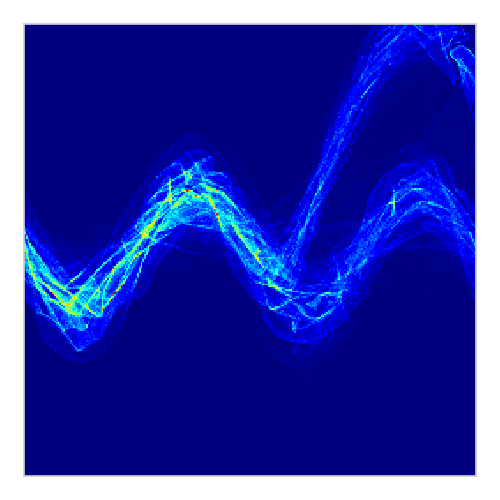}
    
    \vspace*{-0.05cm}

    \caption*{\small $d_z=2$}
    
    \vspace*{-0.2cm}
    
    \caption*{\scriptsize}
    
\end{subfigure}
\begin{subfigure}[b]{0.32\textwidth} %
    \includegraphics[width=\linewidth]{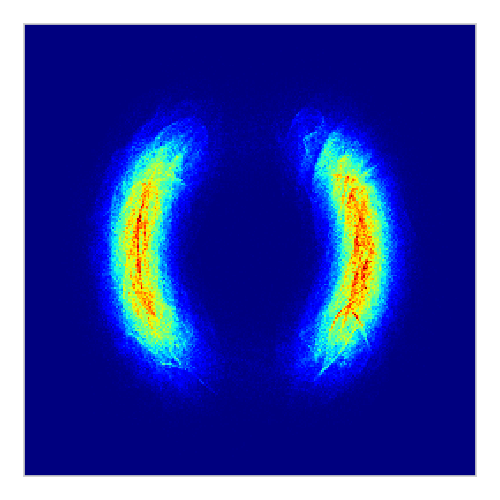}
    \includegraphics[width=\linewidth]{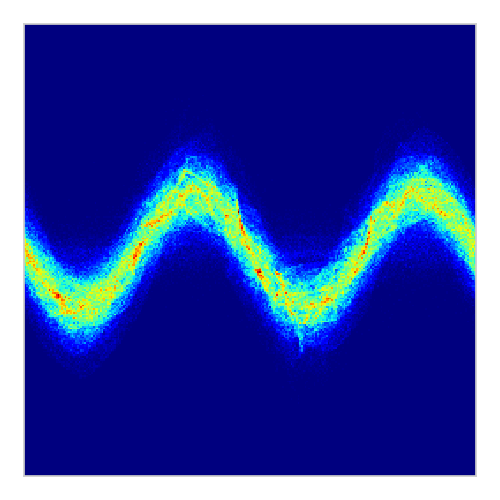}
    \includegraphics[width=\linewidth]{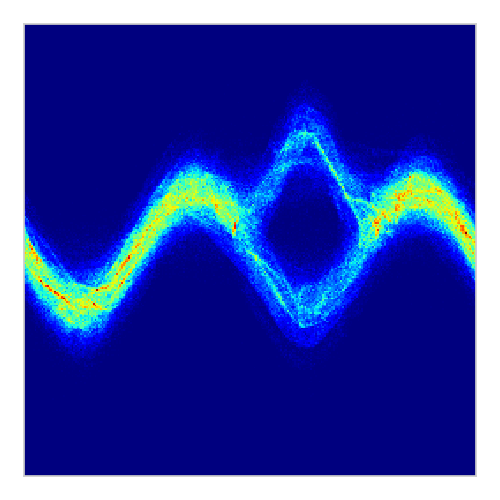}
    \includegraphics[width=\linewidth]{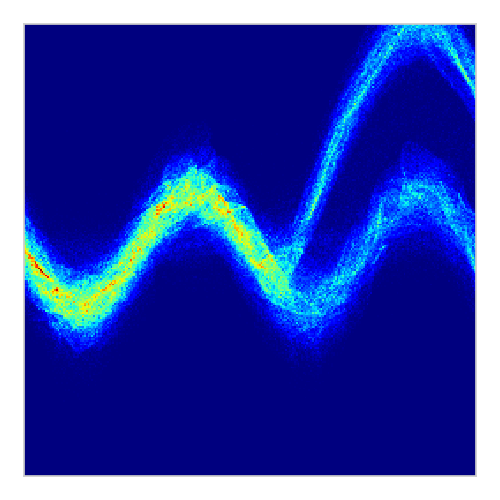}
    
    \vspace*{-0.05cm}

    \caption*{\small  $d_z=3$}
    
    \vspace*{-0.2cm}
    
    \caption*{\scriptsize}
    
\end{subfigure}
\begin{subfigure}[b]{0.32\textwidth} %
    \includegraphics[width=\linewidth]{figs/fit/po1-ardae-im-nd5-nstd10-mnh3.png}
    \includegraphics[width=\linewidth]{figs/fit/po2-ardae-im-nd5-nstd10-mnh3.png}
    \includegraphics[width=\linewidth]{figs/fit/po3-ardae-im-nd5-nstd10-mnh3.png}
    \includegraphics[width=\linewidth]{figs/fit/po4-ardae-im-nd5-nstd10-mnh3.png}
    
    \vspace*{-0.05cm}

    \caption*{\small $d_z=10$}
    
    \vspace*{-0.2cm}
    
    \caption*{\scriptsize}
    
\end{subfigure}

\vspace*{-0.35cm}

\caption{}

\end{subfigure}

\vspace*{-0.4cm}

\caption{
Fitting energy functions with implicit model using AR-DAE. 
(a) Target energy functions. 
(b) Varying number of AR-DAE updates per model update. %
(c) Varying the dimensionality of the noise source $d_z$. 
}
\label{fig:result-unnorm-density-estimation-num-updates}
\end{figure}

\subsection{Effect of the noise dimension of implicit model}
In this section, we study the effect of varying the dimensionality of the noise source of the implicit distribution. 
We use the same experiment settings in the previous section. %
In Figure \ref{fig:result-unnorm-density-estimation-num-updates} (right panel), we see that the generator has a degenerate distribution when $d_z=2$, and the degeneracy can be remedied by increasing $d_z$. %

\section{Experiment: variational autoencoders}
\label{appendix:vae}

\subsection{VAE with the entropy gradient approximator}
Let $p_{\omega}(x|z)$ be the conditional likelihood function parameterized by $\omega$ and $p(z)$ be the the prior distribution. 
We let $p(z)$ be the standard normal. As described in Section \ref{sec:vae}, we would like to maximize the ELBO (denoted as $\cL_{ELBO}$) by jointly training $p_{\omega}$ and the amortized variational posterior $q_{\phi}(z|x)$. Similar to Appendix \ref{sec:appendix-entropy-gradient}, the posterior $q_{\phi}(z|x)$ can be induced by a mapping $g_{\phi}: \epsilon, x \mapsto z$ with a prior $q(\epsilon)$ that does not depend on the parameter $\phi$. %
The gradient of $\cL_{ELBO}$ wrt the parameters of the posterior can be written as,
\eq{
\label{eq:elbo-grad-ardae}
\nabla_{\phi} \cL_{ELBO}(q) = \eE_{
    \substack{z \sim q_{\phi}(z|x) \\
    x \sim p_{\tt{data}}(x)}} 
    \lbs \lbs
    \nabla_{z} \log p_{\omega}(x, z) - \nabla_{z} \log q_{\phi}(z|x)
    \rbs^{\intercal} \mathbf{J}_{\phi}g_{\phi}(\epsilon, x) \rbs
.
}

We plug in AR-DAE to approximate the gradient of the log-density, and draw a Monte-Carlo sample of the following quantity to estimate the gradient of the ELBO
\eq{
\label{eq:elbo-grad-ardae-approx}
\hat{\nabla}_{\phi} \cL_{ELBO}(q) \doteq \eE_{
    \substack{z \sim q_{\phi}(z|x) \\
    x \sim p_{\tt{data}}(x)}} 
    \lbs \lbs
    \nabla_{z} \log p_{\omega}(x, z) - f_{ar,\theta}(z;x, \sigma)|_{\sigma=0}
    \rbs^{\intercal} \mathbf{J}_{\phi}g_{\phi}(\epsilon, x) \rbs
,
}

\subsection{AR-DAE}
\label{app:subsec:ardae-implementations}
To approximate $\nabla_{z} \log q_{\phi}(z|x)$, we condition AR-DAE on both the input $x$ as well as the noise scale $\sigma$. %
We also adaptively choose the prior variance $\delta^2$ for different data points instead of fixing it to be a single value.

In addition, we make the following observations.
(1) The posteriors $q_{\phi}$ are usually not centered, but the entropy gradient approximator only needs to model the dispersion of the distribution. %
(2) The variance of the approximate posterior can be very small during training, which %
might pose a challenge for optimization. 
To remedy these, we modify the input of AR-DAE to be $\tilde{z} \doteq s (z - b(x))$, where $s$ is a scaling factor and $b(x)$ is a pseudo mean. Ideally, we would like to set $b(x)$ to be $\eE_{q(z|x)}[z]$. %
Instead, we let 
$b(x) \doteq g(0, x)$, as $0$ is the mode/mean of the noise source. %
The induced distribution of $\tilde{z}$ will be denoted by $q_{\phi}(\tilde{z}|x)$. By the change-of-variable density formula, we have %
$\nabla_{z} \log q(z|x) = s \nabla_{\tilde{z}} \log q(\tilde{z}|x)$.
This allows us to train AR-DAE with a better-conditioned distribution and the original gradient can be recovered by rescaling. 

In summary, we optimize the following objective
\eq{
\label{eq:ardae++}
\cL_{\tt{ar}}\lbp f_{ar} \rbp
= \eE_{\substack{
        x \sim p(x) \\
        \tilde{z} \sim q(\tilde{z}|x) \\
        u \sim N(0, I) \\
        \sigma|x \sim N(0, \delta(x)^2) \\
        }}
    \lbs \lbV u + \sigma f_{ar}(\tilde{z} + \sigma u; x, \sigma) \rbV^2 \rbs   
.
}
where $\delta(x) \doteq \delta_{\tt{scale}} S_{z|x}$ and $S_{z|x}$ is sample standard deviation of $z$ given $x$. We use $n_z$ samples per data to estimate $S_{z|x}$. $\delta_{\tt{scale}}$ is chosen as hyperparameter.

In the experiments, we either directly parameterize the residual function of AR-DAE or indirectly parameterize it as the gradient of some scalar-function. We parameterize $f_{ar}(\tilde{z}; x, \sigma)$ as a multi-layer perceptron (MLP). Latent $z$ and input $x$ are encoded separately and then concatenated with $\sigma$ (denoted by "mlp-concat" in Table \ref{table:vae-hyperparameters}). The MLP encoders have $m_{\tt{enc}}$ hidden layers. The concatenated representation is fed into a fully-connected neural network with $m_{\tt{fc}}$ hidden layers. Instead of encoding the input $x$ directly, we either use a hidden representation of the variational posterior $q$ or $b(x)$. We use $d_h$ hidden units for all MLPs. We stress that the learning signal from $\cL_{\tt{ar}}\lbp f_{ar} \rbp$ is not backpropagated to the posterior.

\begin{algorithm}[ht]
\caption{VAE AR-DAE}
\label{alg:vae_ardae}
\begin{algorithmic}
\STATE {\bfseries Input:}
Dataset $\cD$;
mini-batch size $n_{\tt{data}}$;
sample size $n_{z}$;%
prior variance $\delta^2$;
learning rates $\alpha_{\theta}$ and $\alpha_{\phi, \omega}$ \\
\STATE Initialize encoder and decoder $p_\omega(x|z)$ and $q_\phi(z|x)$ %
\STATE Initialize AR-DAE $f_{ar, \theta}(z|x)$
\REPEAT
\STATE Draw $n_{\tt{data}}$ datapoints from $\cD$
    \FOR{$k = 0 \dots N_d$}
        \STATE Draw $n_{z}$ latents per datapoint from $z \sim q_{\phi}(z|x)$ %
        \STATE $\delta_i \gets  \delta_{\tt{scale}} S_{z|x_i} \textrm{ for } i=1,\dots, n_{\tt{data}}$
        \STATE Draw $n_{\sigma}$ number of $\sigma_i$s per $z$ from $\sigma_i \sim N(0,\delta_i^2)$ %
        \STATE Draw $n_{\tt{data}}n_{z}n_{\sigma}$ number of $u$s from $u \sim N(0, I)$
        \STATE Update $\theta$ using gradient $\nabla_{\theta} \mathcal{L}_{f_{ar}}$ with learning rate $\alpha_\theta$
    \ENDFOR
    \STATE $z \sim q_{\phi}(z|x)$
    \STATE Update $\omega$ using gradient $\nabla_{\omega} \mathcal{L}_{ELBO}$ with learning rate $\alpha_{\phi, \omega}$
    \STATE Update $\phi$ using gradient $\hat{\nabla}_{\phi} \mathcal{L}_{ELBO}$ with learning rate $\alpha_{\phi, \omega}$, whose entropy gradient is approximated using $f_{ar, \theta}(z|x)$.
\UNTIL{Until some stopping criteria}
\end{algorithmic}
\end{algorithm}

\subsection{Experiments}
We summarize the architecture details and hyperparameters in Table~\ref{table:vae-networks} and \ref{table:vae-hyperparameters}, respectively.
\label{app:subsec:vae-experiments}
\para{Mixture of Gaussian experiment} 
For the MoG experiment, we use 25 Gaussians centered on an evenly spaced 5 by 5 grid on $[-4, 4] \times [-4, 4]$ with a variance of $0.1$. 
Each model is trained for 16 epochs: approximately 4000 updates with a minibatch size of 512.

For all experiments, we use a two-hidden-layer MLP to parameterize the conditional diagonal Gaussian $p(x|z)$. For the implicit posterior $q$, the input $x$ and the $d_{\epsilon}$-dimensional noise are separately encoded with one fully-connected layer, and then the concatenation of their features will be fed into a two-hidden-layer MLP to generate the 2-dimensional latent $z$. 
The size of the noise source $\epsilon$ in the implicit posterior, \emph{i.e.} $d_{\epsilon}$, is set to 10.

\para{MNIST}
We first describe the details of the network architectures and then continue to explain training settings. For the {\it MLP} experiments, we use a one-hidden-layer MLP for the diagonal Gaussian decoder $p(x|z)$. %
For the diagonal Gaussian posterior $q(z|x)$. aka vanilla VAE, input $x$ is fed into a fully-connected layer and then the feature is later used to predict the mean and diagonal component of the covariance matrix of the multivariate Gaussian distribution. For the hierarchical posterior, both $q(z_0|x)$ and $q(z|z_0,x)$ are one-hidden-layer MLPs with diagonal Gaussian similar to the vanilla VAE. For the implicit posterior, the input is first encoded and then concatenated with noise before being fed into another MLP to generate $z$.

For {\it Conv}, the decoder starts with a one-fully connected layer followed by three deconvolutional layers. The encoder has three convolutional layers and is modified depending on the types of the variational posteriors, similar to {\it MLP}. For {\it ResConv}, five convolutional or deconvolutional layers with residual connection are used for the encoder and the decoder respectively. %

Following \citet{maaloe2016auxiliary, ranganath2016hierarchical}, when the auxiliary variational method (HVI aux) is used to train the hierarchical posterior, the variational lower bound is defined as,
we maximize the following lower bound to train the hierarchical variational posterior with auxiliary variable (HVI aux)
\eqst{
\log p(x) 
\ge \eE_{z \sim q(z|x)} \lbs \log p(x, z) - \log q(z|x) \rbs
\ge \eE_{\substack{
z_0 \sim q(z_0|x) \\
z \sim q(z|z_0, x) \\
}} \lbs \log p(x, z) - \log q(z_0|x) - \log q(z|z_0, x) + \log h(z_0 | z, x) \rbs
.
}

For the dynamically binarized MNIST dataset, we adopt the experiment settings of \citet{MeschederNG17/icml}. The MNIST data consists of 50k train, 10k validation, and 10k test images. In addition to the original training images, randomly selected 5k validation images are added to the training set.
Early stopping is performed based on the evaluation on the remaining 5k validation data points. The maximum number of iterations for the training is set to 4M.

For the statically binarized MNIST dataset, we use the original data split. Early stopping as well as hyperparameter search are performed based on the estimated log marginal probability on the validation set.  We retrain the model with the selected hyperparameters with the same number of updates on the combination of the train+valid sets, and report the test set likelihood. We also apply polyak averaging \citep{polyak1992acceleration}. 

We evaluate $\log p(x)$ of the learned models using importance sampling \cite{BurdaGS15/iclr} (with $n_{\textrm{eval}}$ samples). For the baseline methods, we use the learned posteriors as proposal distributions to estimate the log probability.
When a posterior is trained with AR-DAE, we first draw $n_{\textrm{eval}}$ $z$'s from the posterior given the input $x$, and then use the sample mean and covariance matrix to construct a multivariate Gaussian distribution. We then use this Gaussian distribution as the proposal.

\section{Experiment: entropy-regularized reinforcement learning}
\label{appendix:reinforcement-learning}

\subsection{Soft actor-critic}
\label{appendix:subsec:sac}
\para{Notation} We consider an infinite-horizon Markov decision process (MDP) defined as a tuple $\lbp \States, \Actions, \Rewards, \penv, \gamma \rbp$ \citep{sutton1998introduction}, 
where $\States$, $\Actions$, $\Rewards$ are the spaces of state, action and reward, respectively, $\penv(s_{t+1}|s_t,a_t)$ and $\penv(s_0)$ represent the transition probability and the initial state distribution,
$r(s_t,a_t)$ is a bounded reward function, and $\gamma$ is a discount factor. We write $\tau$ as a trajectory resulting from interacting with the environment under some policy $\pi(a_t|s_t)$. %

The entropy-regularized reinforcement learning \citep{ziebart2010modeling} is to learn a policy $\pi(a_t|s_t)$ that maximizes the following objective;
\eq{
\label{eq:ent-regul-rl}
\cL(\pi) = \eE_{\tau \sim \pi, \penv} \lbs \sum_{t=0}^{\infty} \gamma^{t} \lbp r(s_t, a_t) + \alpha H(\pi(\cdot | s_t)) \rbp \rbs
,
}
where $\alpha$ is an entropy regularization coefficient.
We define a soft state value function $V^{\pi}$ and s soft Q-function $Q^{\pi}$ as follows,
\begin{gather*}
V^{\pi}(s) = \eE_{\tau \sim \pi, \penv} \lbs \sum_{t=0}^{\infty} \gamma^{t} \lbp r(s_t, a_t) + \alpha H(\pi(\cdot | s_t)) \rbp \middle\vert s_0 = s \rbs \\
Q^{\pi}(s, a) = \eE_{\tau \sim \pi, \penv} \lbs r(s_t, a_t) + \sum_{t=1}^{\infty} \gamma^{t} \lbp r(s_t, a_t) + \alpha H(\pi(\cdot | s_t)) \rbp \middle\vert s_0 = s, a_0 = a \rbs
.
\end{gather*}
By using these definitions, we can rewrite $V^{\pi}$ and $Q^{\pi}$ as
$V^{\pi}(s) = \eE_{a \sim \pi} \lbs Q^{\pi}(s, a) \rbs + \alpha H(\pi(\cdot|s))$ and
$Q^{\pi}(s) = \lbs r(s, a) + \eE_{s' \sim \penv} \gamma V^{\pi}(s') \rbs$.

\para{Soft actor-critic}
One way to maximize (\ref{eq:ent-regul-rl}) is to minimize the following KL divergence,
\eqst{
\pi_{\tt{new}} = \argmin_{\pi} \kld \lbp \pi(\cdot | s_t) \middle\Vert \f{\exp \lbp Q^{\pi_{\tt{old}}}(s_t, \cdot) \rbp}{Z^{\pi_{\tt{old}}}(s_t)} \rbp %
,
}
where $Z^{\pi_{\tt{old}}}(s_t)$ is the normalizing constant $\int \exp \lbp Q^{\pi_{\tt{old}}}(s_t, a) \rbp da$. \citet{HaarnojaZAL18/icml} show that for finite state space the entropy-regularized expected return will be non-decreasing if the policy is updated by the above update rule. In practice, however, we do not have access to the value functions, %
so \citet{HaarnojaZAL18/icml} propose 
to update the policy by first approximating $Q^{\pi_{\tt{old}}}$ and $V^{\pi_{\tt{old}}}$ by some 
parametric functions
$Q_{\omega}$ and $V_{\nu}$, 
and training the policy by minimizing
\eqst{
\cL(\pi) = \eE_{s_t \sim \cD} \lbs 
\kld \lbp
\pi(a_t|s_t) \middle\Vert \f{\exp \lbp Q_{\omega}(s_t, \cdot) \rbp}{Z_{\omega}(s_t)}
\rbp
\rbs
,
}
where $\cD$ is a replay buffer that stores all the %
past experience. 
The soft Q-function and soft state value function will be trained by minimizing the following objectives,
\begin{gather*}
\cL(V_{\nu}) = \eE_{s_t \sim \cD} \lbs \f{1}{2} \lbp 
    V_{\nu}(s_t) - \eE_{\small a_t \sim \pi} \lbs Q_{\omega}(s_t, a_t) - \alpha \log \pi(a_t|s_t) \rbs \rbp^2 \rbs \nn\\
\cL(Q_{\omega}) = \eE_{s_t, a_t \sim \cD} \lbs \f{1}{2} \lbp
    Q_{\omega}(s_t, a_t) - \hat{Q}(s_t, a_t) 
    \rbp^2 \rbs
,
\end{gather*}
where $\hat{Q}(s_t, a_t) \doteq r(s_t, a_t) + \gamma \eE_{s_{t+1} \sim \penv} [V_{\bar{\nu}}(s_{t+1})]$ and $V_{\bar{\nu}}$ is a target value network. For the target value network, SAC follows \citet{MnihKSRVBGRFOPB15/nature}: $V_{\bar{\nu}}$ is defined as a polyak-averaged model \citep{polyak1992acceleration} of $V_{\nu}$. Note that $V_{\nu}$ is inferred from $Q_{\omega}$ via Monte Carlo, \emph{i.e.} $V_{\nu}(s_t) \doteq  Q_{\omega}(s_t, a_t) - \alpha \log \pi(a_t|s_t)$ where $a_t \sim \pi(a_t| s_t)$. Moreover, we follow the common practice to use the clipped double Q-functions \citep{hasselt2010double, FujimotoHM18/icml} in our implementations.

\subsection{SAC-AR-DAE and its implementations}
\label{app:sac-ar-dae}
\para{Main algorithm}
Our goal is to train an arbitrarily parameterized policy within the SAC framework. We apply AR-DAE to approximate the training signal for policy. Similar to the implicit posterior distributions in the VAE experiments, the policy consists of a simple tractable noise distribution $\pi(\epsilon)$ and a mapping $g_{\phi}: \epsilon, s \mapsto a$. The gradient of $\cL(\pi)$ wrt the policy parameters can be written as
\eqst{
\nabla_{\phi} \cL(\pi) = \eE_{
    \substack{s_t \sim \cD \\
    \epsilon \sim \pi}} 
    \lbs \lbs \nabla_{a} \log \pi_{\phi} (a|s_t)|_{a=g_{\phi}(\epsilon, s_t)} -  \nabla_{a} Q_{\omega}(s_t, a)|_{a=g_{\phi}(\epsilon, s_t)} \rbs^{\intercal} \mathbf{J}_{\phi}g_{\phi}(\epsilon, s_t) \rbs
.
}

Let $f_{ar,\theta}$ be AR-DAE which approximates $\nabla_{a} \log \pi_{\phi} (a|s)$ trained using Equation (\ref{eq:ardae++}). Specifically for the SAC experiment, AR-DAE is indirectly parameterized as the gradient of an unnormalized log-density function $\psi_{ar, \theta}: a, s, \sigma \mapsto \eR$ as in,
\eqst{
f_{ar,\theta}(a;s, \sigma) \doteq \nabla_{a} \psi_{ar, \theta}(a;s, \sigma)
.
}
As a result, $\log \pi(a|s)$ can also be approximated by using $\psi_{ar, \theta}$: $\log \pi(a|s) \approx \psi_{ar, \theta}(a;s, \sigma)|_{\sigma=0} - \log Z_{\theta}(s)$, where $Z_{\theta}(s) = \int \exp \lbp \psi_{ar, \theta}(a;s, \sigma)|_{\sigma=0} \rbp da$.

Using AR-DAE, we can modify the objective function $\cL(V_{\nu})$ to be
\eqst{
\hat{\cL}(V_{\nu}) = \eE_{s_t \sim \cD} \lbs \f{1}{2} \lbp 
    V_{\nu}(s_t) - \eE_{\small a_t \sim \pi} \lbs Q_{\omega}(s_t, a_t) - 
    \psi_{ar, \theta}(a_t;s_t, \sigma)|_{\sigma=0} \rbs 
    - \log Z_{\theta}(s_t)
    \rbp^2 \rbs
.
}
The same applies to $\cL(Q_{\omega})$.
We also use the polyak-averaged target value network and one-sample Monte-Carlo estimate as done in SAC. Finally, the gradient signal for the policy can be approximated using AR-DAE:
\eqst{
\hat{\nabla}_{\phi} \cL(\pi) \doteq \eE_{
    \substack{s_t \sim \cD \\
    \epsilon \sim \pi}} 
    \lbs \lbs
    f_{ar,\theta}(g_{\phi}(\epsilon, s_t);s_t, \sigma)|_{\sigma=0} -  \nabla_{a} Q_{\omega}(s_t, a)|_{a=g_{\phi}(\epsilon, s_t)} \rbs^{\intercal} \mathbf{J}_{\phi}g_{\phi}(\epsilon, s_t) \rbs
.
}
We summarize all the details in Algorithm \ref{alg:sac_ardae}.

\begin{algorithm}%
\caption{SAC-AR-DAE}
\label{alg:sac_ardae}
\begin{algorithmic}
\STATE {\bfseries Input:} Mini-batch size $n_{\tt{data}}$; replay buffer $\cD$; number of epoch 
$T$;
learning rates $\alpha_{\theta},\alpha_{\phi},\alpha_{\omega}, \alpha_{\nu}$\\
\STATE Initialize value function $V_\nu(s)$, critic $Q_\omega(s,a)$, policy $\pi_{\phi}(a|s)$, and AR-DAE $f_{ar, \theta}(a|s)$
\STATE Initialize replay buffer $\cD \gets \varnothing$  
\FOR{$\text{epoch}=1,...,T$} 
\STATE Initialize a state from $s_0 \sim \penv(s_0)$ %
\FOR{$t = 0 \dots $}
\STATE $a \sim \pi_{\phi}(.|s_t)$
\STATE $(r_t, s_{t+1}) \sim \penv(\cdot|s_t,a_t)$ %
\STATE $\cD \gets \cD \cup \{(s_t,a_t,r_t,s_{t+1})\}$%
\FOR{each learning step}
    \STATE Draw $n_{\tt{data}}$ number of $(s_t,a_t,r_t,s_{t+1})$s from $\cD$
    \FOR{$k = 0 \dots N_d$}
        \STATE Draw $n_{a}$ actions per state from $a \sim \pi_{\phi}(a|s)$ %
        \STATE $\delta_i \gets \delta_{\tt{scale}} S_{a|s_i} \textrm{ for } i=1,\dots, n_{\tt{data}}$
        \STATE Draw $n_{\sigma}$ number of $\sigma_i$s per $a$ from $\sigma_i \sim N(0,\delta_i^2)$ %
        \STATE Draw $n_{\tt{data}}n_{a}n_{\sigma}$ number of $u$s from $u \sim N(0, I)$
        \STATE Update $\theta$ using gradient $\nabla_{\theta} \mathcal{L}_{f_{ar}}$ with learning rate $\alpha_\theta$
    \ENDFOR
    \STATE Update $\nu$ using gradient $\nabla_{\nu} \hat{\cL}_{V}$ with learning rate $\alpha_\nu$ %
    \STATE Update $\omega$ using gradient $\nabla_{\omega} \hat{\cL}_{Q}$ with learning rate $\alpha_\omega$ %
    \STATE Update $\phi$ using gradient $\hat{\nabla}_{\phi} \mathcal{L}_{\pi}$ which is approximated with $f_{ar, \theta}(a|s)$
    \STATE $\bar{\nu} \gets \tau \nu + (1-\tau)\bar{\nu}$ %
\ENDFOR
\ENDFOR
\ENDFOR
\end{algorithmic}
\end{algorithm}

\para{Bounded action space}
The action space of all of our environments is an open cube $(-1, 1)^{d_a}$,
where $d_a$ is the dimensionality of the action.
To implement the policy, we apply the hyperbolic tangent function. That is, $a := \tanh(g_{\phi}(\epsilon, s_t))$, where the output of $g_{\phi}$ (denoted as $\tilde{a}$) is in $(-\infty, \infty)$. %
Let $\tilde{a}_{i}$ be the $i$-th element of $\tilde{a}$. 
By the change of variable formula, $\log \pi (a|s) = \log \pi(\tilde{a}|s) - \sum_{i=1}^{d_a} \log (1- \tanh^2(\tilde{a}_{i}))$. 

In our experiments, we train AR-DAE on the pre-$\tanh$ action $\tilde{a}$. This implies that AR-DAE approximate $\nabla_{\tilde{a}} \log \pi(\tilde{a}|s)$. 
We correct the change of volume induced by the tanh using %
\eqst{
\nabla_{\tilde{a}} \log \pi(a|s) = \nabla_{\tilde{a}} \log \pi(\tilde{a}|s) + 2\tanh(\tilde{a})
.
}
To sum up, 
the update of the policy follows the approximated gradient
\eqst{
\hat{\nabla}_{\phi} \cL(\pi) \doteq \eE_{
    \substack{s_t \sim \cD \\
    \epsilon \sim \pi}} 
    \lbs \lbs
    f_{ar,\theta}(g_{\phi}(\epsilon, s_t);s_t, \sigma)|_{\sigma=0} + 2\tanh(g_{\phi}(\epsilon, s_t)) -  \nabla_{\tilde{a}} Q_{\omega}(s_t, \tanh(\tilde{a}))|_{\tilde{a}=g_{\phi}(\epsilon, s_t)} \rbs^{\intercal} \mathbf{J}_{\phi}g_{\phi}(\epsilon, s_t) \rbs
.
}

\para{Estimating normalizing constant}
In order to train SAC-AR-DAE in practice, efficient computation of $\log Z_{\theta}(s)$ is required. 
We propose to estimate the normalizing constant \citep{geyer1991reweighting} using importance sampling. %
Let $h(a|s)$ be the proposal distribution.
We compute the following (using the log-sum-exp trick to ensure numerical stability)
\eqst{
\log Z_{\theta}(s) 
&= \log \int \exp \lbp \psi_{ar, \theta}(a;s, \sigma)|_{\sigma=0} \rbp da \nn\\
&= \log \eE_{a \sim h} \lbs \exp\lbp\psi_{ar, \theta}(a;s, \sigma)|_{\sigma=0} - \log h(a|s) \rbp \rbs \nn\\
&\approx \log \f{1}{N_Z}\sum_{j}^{N_Z} \lbs \exp\lbp\psi_{ar, \theta}(a_j;s, \sigma)|_{\sigma=0} - \log h(a_j|s) - A \rbp \rbs + A %
,
}
where $a_j$ is the $j$-th action sample from $h$ and $A := \max_{a_j} \exp\lbp\psi_{ar, \theta}(a_j;s, \sigma)|_{\sigma=0} - \log h(a_j|s) \rbp$.
For the proposal distribution, we use $h(a|s) \doteq N(\mu(s), c I)$, where $\mu(s) \doteq \psi_{ar, \theta}(g_{\phi}(\epsilon, s);s, \sigma)|_{\epsilon=0, \sigma=0}$ and $c$ is some constant. We set $c$ to be $\log c = -1$.

\para{Target value calibration}
In order to train the Q-function more efficiently, we calibrate its target values. Training the policy only requires estimating the gradient of the Q-function wrt the action, not the value of the Q-function itself. This means that while optimizing $Q_{\omega}$ (and $V_{\nu}$), we can subtract some constant from the true target to center it. 
In our experiment, this calibration is applied when we use one-sample Monte-Carlo estimate and the polyak-averaged Q-network $Q_{\bar{\omega}}$. That is, $\cL(Q_{\omega})$ can be rewritten as,
\eqst{
\cL(Q_{\omega}) = \eE_{\substack{s_t, a_t, s_{t+1} \sim \cD\\ a_{t+1} \sim \pi}} \lbs \f{1}{2} \lbp
    Q_{\omega}(s_t, a_t) + B - 
    r(s_t, a_t) - \gamma \lbp Q_{\bar{\omega}}(s_{t+1}, a_{t+1}) - \alpha \log \pi(a_{t+1}|s_{t+1})
    \rbp
    \rbp^2 \rbs
.
}
where $B$ is a running average of the expected value of $\gamma \alpha \log \pi(a|s)$ throughout training. %

\para{Jacobian clamping}
In addition, %
we found that the implicit policies can potentially collapse to point masses. %
To mitigate this, we regularize the implicit distributions by controlling the Jacobian matrix of the policy wrt the noise source as in \citet{odena2018generator, kumar2020regularized}, aka \emph{Jacobian clamping}.
The goal is to ensure all singular values of Jacobian matrix of pushforward mapping to be higher than some constant.
In our experiments, we follow the implementation of \citet{kumar2020regularized}: (1) stochastic estimation of the singular values of Jacobian matrix at every noise, and the Jacobian is estimated by finite difference approximation, and (2) use of the penalty method \citep{bertsekas2016nonlinear} to enforce the constraint. The resulting regularization term is
\begin{align*}
    \mathcal{L}_{\textrm{reg}}(\pi) = \eE_{
    \substack{s_t \sim \cD \\
    \epsilon \sim \pi \\
    v \sim N(0, I)}} \left[
    \min\left(
    \frac{\Vert g_{\phi}(\epsilon + \xi v, s_t) - g_{\phi}(\epsilon, s_t) \Vert^{2}_{2}}{\xi^2 \Vert v\Vert^2}
    - \eta, 0 \right)^2 \right],
\end{align*}
where $\eta, \xi > 0$, and $n_{\tt{perturb}}$ number of the perturbation vector $v$ is sampled. We then update policy $\pi$ with $\hat{\nabla}_{\phi} \cL(\pi) + \lambda \nabla_{\phi} \mathcal{L}_{\textrm{reg}}(\pi)$ where $\lambda$ is increased throughout training. We set $\lambda = 1+i^{\nu}/1000$ at $i$-th iteration and $\nu \in [1.1, 1.3]$. %

\begin{figure}%
\centering

\vspace*{-0.3cm}

\begin{subfigure}[b]{0.588\textwidth}
\centering
\hspace*{0.4cm}\includegraphics[width=\textwidth]{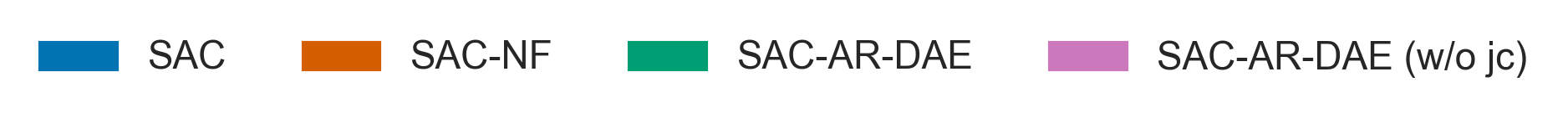}
\end{subfigure}

\vspace*{-0.2cm}

\begin{subfigure}[b]{\textwidth}
\centering
\hspace*{-0.3cm}\includegraphics[width=\textwidth]{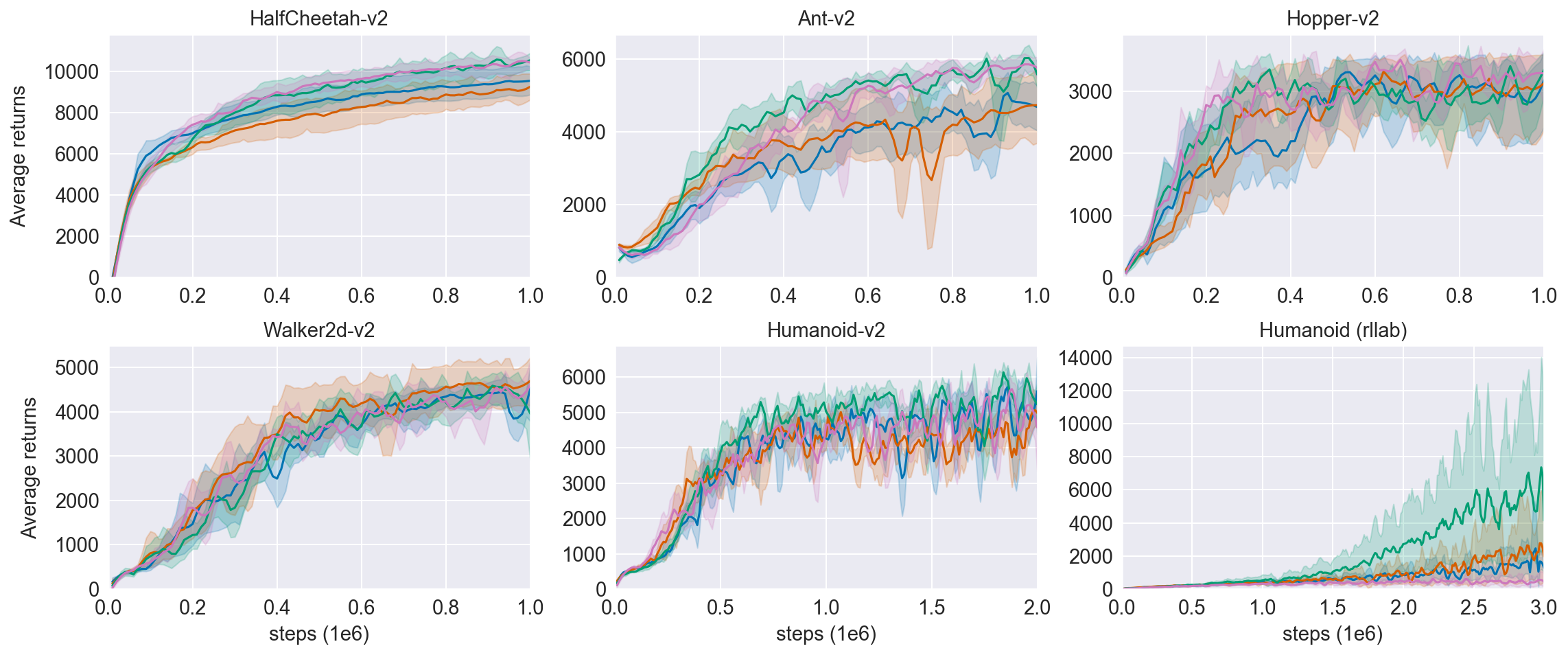}
\end{subfigure}

\vspace*{-0.3cm}

\caption{
Additional results on SAC-AR-DAE, ablating Jacobian clamping regularization on implicit policy distributions in comparison with the rest.
}
\label{fig:results-rl-extended}
\end{figure}

\subsection{Experiments}
For the SAC-AR-DAE experiments, 
aside from the common practice for SAC,
we follow the experiment settings from \citet{MazoureDDHP19}
and
sample from a uniform policy for a fixed number of initial interactions (denoted as {\it warm-up}). We also adopt the same network architecture for the Q-network, discounting factor $\gamma$, entropy regularization coefficient $\alpha$, and target smoothing coefficient $\tau$.
For AR-DAE, we use the same network architecture as VAE. We also rescale the unbounded action $\tilde{a}$ by $s$ for better conditioning. The details of hyperparameters are described in Table \ref{table:rl-hyperparameters}. 

We run five experiments for each environment without fixing the random seed.~For every 10k steps of environment interaction, the average return of the policy is evaluated with 10 independent runs. For visual clarify, the learning curves are smoothed by second-order polynomial filter with a window size of 7 \cite{savitzky1964smoothing}. For each method, we evaluate the maximum average return: we take the maximum of the average return for each experiment and the average of the maximums over the five random seeds. We also report `normalized average return', approximately area under the learning curves: we obtain the numerical mean of the `average returns' over iterates. We run SAC and SAC-NF with the hyperparameters reported in \citet{MazoureDDHP19}. %

\subsection{Additional Experiments}

In addition to the main results in Figure \ref{fig:results-rl} and Table \ref{tab:sac-result-max-main}, we also compare the effect of Jacobian clamping regularization on implicit policy distribution in SAC-AR-DAE. In each environment, the same hyperparameters are used in SAC-AR-DAEs except for the regularization. Our results are presented in Figure \ref{fig:results-rl-extended} and Table \ref{tab:sac-result-max-full}, \ref{tab:sac-result-auc-full}.

The results shows that 
Jacobian clamping regularization improves the performance of SAC-AR-DAE in general, especially for {\it Humanoid-rllab}. %
In {\it Humanoid-rllab}, we observe that implicit policy degenerates to point masses without the Jacobian clamping, potentially due to the error of AR-DAE. However, the Jacobian clamping helps to avoid the degenerate distributions, and the policy facilitates AR-DAE-based entropy gradients.

\begin{table}[H]
\centering
\small
\begin{tabular}{lcccc}
\toprule
 & \multicolumn{1}{c}{SAC} & \multicolumn{1}{c}{SAC-NF} & \multicolumn{1}{c}{SAC-AR-DAE} & \multicolumn{1}{c}{SAC-AR-DAE (w/o jc)} \\
\midrule
HalfCheetah-v2   & 9695 $\pm$ 879 & 9325 $\pm$ 775 & \textbf{10907 $\pm$ 664} & 10677 $\pm$ 374 \\
Ant-v2           & 5345 $\pm$ 553 & 4861 $\pm$ 1091 & \textbf{6190 $\pm$ 128} & 6097 $\pm$ 140 \\
Hopper-v2        & 3563 $\pm$ 119 & 3521 $\pm$ 129 & 3556 $\pm$ 127 & \textbf{3634 $\pm$ 45} \\
Walker-v2        & 4612 $\pm$ 249 & 4760 $\pm$ 624 & 4793 $\pm$ 395 & \textbf{4843 $\pm$ 521} \\
Humanoid-v2      & 5965 $\pm$ 179 & 5467 $\pm$ 44 & \textbf{6275 $\pm$ 202} & 6268 $\pm$ 77 \\
Humanoid (rllab) & 6099 $\pm$ 8071 & 3442 $\pm$ 3736 & \textbf{10739 $\pm$ 10335} & 761 $\pm$ 413 \\
\bottomrule
\end{tabular}%
\caption{Maximum average return. $\pm$ corresponds to one standard deviation over five random seeds.}
\label{tab:sac-result-max-full}
\end{table}

\begin{table}[H]
\centering
\small
\begin{tabular}{lcccc}
\toprule
 & \multicolumn{1}{c}{SAC} & \multicolumn{1}{c}{SAC-NF} & \multicolumn{1}{c}{SAC-AR-DAE} & \multicolumn{1}{c}{SAC-AR-DAE (w/o jc)} \\
\midrule
HalfCheetah-v2   & 8089 $\pm$ 567 & 7529 $\pm$ 596 & 8493 $\pm$ 602 & \textbf{8636 $\pm$ 307} \\
Ant-v2           & 3280 $\pm$ 553 & 3440 $\pm$ 656 & \textbf{4335 $\pm$ 241} & 4015 $\pm$ 363 \\
Hopper-v2        & 2442 $\pm$ 426 & 2480 $\pm$ 587 & 2631 $\pm$ 160 & \textbf{2734 $\pm$ 194} \\
Walker-v2        & 3023 $\pm$ 271 & \textbf{3317 $\pm$ 455} & 3036 $\pm$ 271 & 3094 $\pm$ 209 \\
Humanoid-v2      & 3471 $\pm$ 505 & 3447 $\pm$ 260 & \textbf{4215 $\pm$ 170} & 3808 $\pm$ 137 \\
Humanoid (rllab) & 664 $\pm$ 321 & 814 $\pm$ 630 & \textbf{2021 $\pm$ 1710} & 332 $\pm$ 136 \\
\bottomrule
\end{tabular}%
\caption{Normalized average return. $\pm$ corresponds to one standard deviation over five random seeds.%
}
\label{tab:sac-result-auc-full}
\end{table}

\section{Improved techniques for training AR-DAE and implicit models}
\label{app:heuristics}
In order to improve and stabilize the training of both the generator and AR-DAE,
we explore multiple heuristics. 

\subsection{AR-DAE}
\para{Activity function} 
During preliminary experiments, we observe that \emph{smooth activation functions} are crucial in parameterizing AR-DAE as well as the residual form of regular DAE. We notice that ReLU gives less reliable log probability gradient for low density regions.

\para{Number of samples and updates}
In the VAE and RL experiments, it is important to keep AR-DAE up-to-date with the generator (i.e. posterior and policy). %
As discussed in Appendix \ref{appendix:energy-fitting}, we found that increasing the number of AR-DAE updates helps a lot. 
Additionally, %
we notice that increasing $n_{z}$ is more helpful than increasing $n_{\tt{data}}$ given $n_{\tt{data}}n_{z}$ is fixed.

\para{Scaling-up and zero-centering data}
To avoid using small learning rate for AR-DAE in the face of sharp distributions with small variance, %
we choose to scale up the input of AR-DAE.
As discussed in Appendix \ref{app:subsec:ardae-implementations}, we also zero-center the latent samples (or action samples) to train AR-DAE. This allows AR-DAE to focus more on modeling the dispersion of the distribution rather than %
where most of the probability mass resides. 

\subsection{Implicit distributions}
\para{Noise source dimensionality}
We note that the implicit density models %
can potentially be degenerate and do not admit a density function. 
For example, in Appendix \ref{appendix:energy-fitting} we show that increasing the dimensionality of the noise source improves the qualities of the implicit distributions. %

\para{Jacobian clamping}
Besides of increasing noise source dimensionality, we can consider Jacobian clamping distributions to prevent implicit posteriors from collapsing to point masses. As pointed out in Appendix \ref{app:sac-ar-dae}, we observe that using this regularization technique can prevent degenerate distributions in practice, as it at least regularizes the mapping locally if its Jacobian is close to singular.

\newpage
\null
\vfill
\begin{table}[!h]
\scriptsize
\centering
\begin{tabular}{cccccc}
\toprule
& \multirow{2}{*}{ $p(x|z)$ } & \multicolumn{4}{c}{ $q(z|x)$ } \\
&  & Common & Gaussian & HVI & implicit \\
\midrule
\begin{tabular}{@{}c@{}}
{\it MLP} \\
\tt{toy}
\end{tabular} & \begin{tabular}{@{}c@{}}
$\lbs 2, 256 \rbs$ \\
$\lbs 256, 256 \rbs \times 2$ \\
$\lbs 256, d_x\rbs \times 2$
\end{tabular} & $\lbs d_x, 256\rbs$ & - & - & \begin{tabular}{@{}c@{}}
$\lbs 256+d_{\epsilon}, 256\rbs$ \\
$\lbs 256, 256\rbs$ \\
$\lbs 256, d_z\rbs$
\end{tabular} \\
\cmidrule(l){1-6}
\begin{tabular}{@{}c@{}}
{\it MLP} \\
\tt{dbmnist}
\end{tabular} & \begin{tabular}{@{}c@{}}
$\lbs d_z, 300 \rbs$ \\
$\lbs 300, d_x\rbs$
\end{tabular} & $\lbs d_x, 300\rbs$ & \begin{tabular}{@{}c@{}}
$\lbs 300, d_z \rbs \times 2$
\end{tabular} & \begin{tabular}{@{}c@{}}
$\lbs 300, d_z \rbs$ $\times$ 2 \\
(or $\lbs 300, d_{z_0}\rbs$ $\times$ 2)
\end{tabular} & \begin{tabular}{@{}c@{}}
$\lbs 300+d_{\epsilon}, 300\rbs$ \\
$\lbs 300, d_z\rbs$
\end{tabular} \\
\cmidrule(l){1-6}
\begin{tabular}{@{}c@{}}
{\it Conv} \\
\tt{dbmnist}
\end{tabular} & \begin{tabular}{@{}c@{}}
$\lbs d_z, 300 \rbs$ \\
$\lbs 300, 512 \rbs$ \\
$\lbs 32, 32, 5 \times 5, 2, 2, \textrm{deconv} \rbs$ \\
$\lbs 32, 16, 5 \times 5, 2, 2, \textrm{deconv} \rbs$ \\
$\lbs 16, 1, 5 \times 5, 2, 2, \textrm{deconv} \rbs$
\end{tabular} & \begin{tabular}{@{}c@{}}
$\lbs 1, 16, 5 \times 5, 2, 2 \rbs$ \\
$\lbs 16, 32, 5 \times 5, 2, 2 \rbs$ \\
$\lbs 32, 32, 5 \times 5, 2, 2 \rbs$
\end{tabular} & \begin{tabular}{@{}c@{}}
$\lbs 512, 800 \rbs$ \\
$\lbs 800, d_z \rbs \times 2$
\end{tabular} & \begin{tabular}{@{}c@{}}
$\lbs 512, 800 \rbs$ \\
$\lbs 800, d_z \rbs \times 2$ \\
(or $\lbs 800, d_(z_0) \rbs \times 2$)
\end{tabular} & \begin{tabular}{@{}c@{}}
$\lbs 512+d_{\epsilon}, 800 \rbs$ \\
$\lbs 800, d_z \rbs$
\end{tabular} \\
\cmidrule(l){1-6}
\begin{tabular}{@{}c@{}}
{\it ResConv} \\
\tt{dbmnist} \\
(or \tt{sbmnist})
\end{tabular} & \begin{tabular}{@{}c@{}}
$\lbs d_z, 450 \rbs$ \\
$\lbs 450, 512 \rbs$ \\
$\lbs \textrm{upscale by 2} \rbs$ \\
$\lbs 32, 32, 3 \times 3, 1, 1, \textrm{res} \rbs$ \\
$\lbs 32, 32, 3 \times 3, 1, 1, \textrm{res} \rbs$ \\
$\lbs \textrm{upscale by 2} \rbs$ \\
$\lbs 32, 16, 3 \times 3, 1, 1, \textrm{res} \rbs$ \\
$\lbs 16, 16, 3 \times 3, 1, 1, \textrm{res} \rbs$ \\
$\lbs \textrm{upscale by 2} \rbs$ \\
$\lbs 16, 1, 3 \times 3, 1, 1, \textrm{res} \rbs$
\end{tabular} & \begin{tabular}{@{}c@{}}
$\lbs 1, 16, 3 \times 3, 2, 1, \textrm{res} \rbs$ \\
$\lbs 16, 16, 3 \times 3, 1, 1, \textrm{res} \rbs$ \\
$\lbs 16, 32, 3 \times 3, 2, 1, \textrm{res} \rbs$ \\
$\lbs 32, 32, 3 \times 3, 1, 1, \textrm{res} \rbs$ \\
$\lbs 32, 32, 3 \times 3, 2, 1, \textrm{res} \rbs$ \\
$\lbs 512, 450, \textrm{res} \rbs$ \\
\end{tabular} & \begin{tabular}{@{}c@{}}
$\lbs 450, d_z \rbs \times 2$
\end{tabular} & \begin{tabular}{@{}c@{}}
$\lbs 450, 450 \rbs$ \\
$\lbs 450, d_z \rbs \times 2$ \\
(or $\lbs 450, d_{z_0} \rbs \times 2)$
\end{tabular} & \begin{tabular}{@{}c@{}}
$\lbs 450+d_{\epsilon}, 450, \textrm{res} \rbs$ \\
$\lbs 450, d_z, \textrm{res} \rbs$
\end{tabular} \\
\bottomrule
\end{tabular}
\caption{Network architectures for the VAE experiments. Fully-connected layers are characterized by [input size, output size], and convolutional layers by [input channel size, output channel size, kernel size, stride, padding]. ``res'' indicates skip connection, aka residual layer \citep{HeZRS16/cvpr}. Deconvolutional layer is marked as ``deconv''.}
\label{table:vae-networks}
\end{table}
\vfill
\begin{table}[!h]
\scriptsize
\centering
\begin{tabular}{cccccccc}
\toprule
& & & \multicolumn{2}{c}{ {\it MLP} } & {\it Conv} & \multicolumn{2}{c}{ {\it ResConv} } \\
& &  & \tt{toy} & \tt{dbmnist} & \tt{dbmnist} & \tt{dbmnist} & \tt{sbmnist} \\
\midrule
\multirow{15}{*}{ AR-DAE } & \multirow{6}{*}{ model } & parameterization & gradient & gradient & gradient & residual & residual \\
& & network & mlp-concat & mlp-concat & mlp-concat & mlp-concat & mlp-concat \\
& & $m_{\tt{fc}}$ & 3 & 5 & 5 & 5 & 5 \\
& & $m_{\tt{enc}}$ & 3 & 5 & 5 & 5 & 5 \\
& & activation & softplus & softplus & softplus & softplus & softplus \\
& & $d_h$ & 256 & 256 & 256 & 512 & 512 \\
& & $s$ & 10000 & 10000 & 10000 & 100 & 100 \\
\cmidrule{2-8}
& \multirow{7}{*}{ learning } & $n_{z}$ & 256 & 625 & 256 & 625 & 625 \\
& & $n_{\tt{data}}$ & 512 & 128 & 128 & 128 & 128 \\
& & $n_{\sigma}$ & 1 & 1 & 1 & 1 & 1 \\
& & $N_d$ & 1 & \{1,2\} & \{1,2\} & 2 & 2 \\
& & $\delta_{\tt{scale}}$ & 0.1 & \{0.1, 0.2, 0.3\} & \{0.1, 0.2, 0.3\} & \{0.1, 0.2, 0.3\} & \{0.1, 0.2, 0.3\} \\
& & optimizer & rmsprop, 0.5 & rmsprop, 0.5 & rmsprop, 0.9 & rmsprop, 0.9 & rmsprop, 0.9 \\
& & learning rate $\alpha_{\theta}$ & 0.0001 & 0.0001 & 0.0001 & 0.0001 & 0.0001 \\
\midrule
\multirow{8}{*}{ Encoder/decoder } & \multirow{3}{*}{ model } & network & mlp & mlp & conv & rescov & rescov \\
& & $d_z$ & 2 & 32 & 32 & 32 & 32 \\
& & $d_{z_0}$ or $d_{\epsilon}$ & 10 & 100 & 100 & 100 & 100 \\
\cmidrule{2-8}
& \multirow{5}{*}{ learning } & $n_{\tt{data}}$ & 512 & 128 & 128 & 128 & 128 \\
& & optimizer & adam, 0.5, 0.999 & adam, 0.5, 0.999 & adam, 0.5, 0.999 & adam ,0.9, 0.999 & adam 0.9, 0.999 \\
& & learning rate $\alpha_{\phi, \omega}$ & 0.0001 & 0.0001 & 0.0001 & \{0.001, 0.0001\} & \{0.001, 0.0001\} \\
& & $\beta$-annealing & no & no & no & \{no, 50000\} & \{no, 50000\} \\
& & \re-train with $\tt{train}$+$\tt{val}$ & no & no & no & no & yes \\
\midrule
\multirow{3}{*}{ Evaluation } & & polyak (decay) & - & no & no & no & 0.998 \\
& & polyak (start interation) & - & no & no & no & \{0, 1000, 5000, 10000\} \\
& & $n_{\textrm{eval}}$ & - & 40000 & 40000 & 20000 & 20000 \\
\bottomrule
\end{tabular}
\caption{Hyperparameters for the VAE experiments. $\tt{toy}$ is the 25 Gaussian dataset. $\tt{dbmnist}$ and $\tt{sbmnist}$ are dynamically and statically binarized MNIST, respectively.}
\label{table:vae-hyperparameters}
\end{table}
\vfill

\newpage
\null
\vfill
\begin{table}[!h]
\scriptsize
\centering
\begin{tabular}{ccccccccc}
\toprule
& & & HalfCheetah-v2 & Ant-v2 & Hopper-v2 & Walker-v2 & Humanoid-v2 & Humanoid (rllab) \\
\midrule
\multirow{14}{*}{AR-DAE} & \multirow{7}{*}{model} & parameterization & gradient & gradient & gradient & gradient & gradient & gradient \\
& & network & mlp & mlp & mlp & mlp & mlp & mlp \\
& & $m_{\tt{fc}}$ & 5 & 5 & 5 & 5 & 5 & 5 \\
& & $m_{\tt{enc}}$ & 5 & 5 & 0 & 1 & 1 & 1 \\
& & activation & elu & elu & elu & elu & elu & elu \\
& & $d_h$ & 256 & 256 & 256 & 256 & 256 & 256 \\
& & $s$ & 10000 & 10000 & 10000 & 10000 & 10000 & 10000 \\
\cmidrule{2-9}
& \multirow{7}{*}{learning} & $n_{a, \tt{dae}}$ & 128 & 64 & 128 & 128 & 64 & 64 \\
& & $n_{\tt{data}}$ & 256 & 256 & 256 & 256 & 256 & 256 \\
& & $n_{\sigma}$ & 1 & 1 & 1 & 1 & 4 & 4 \\
& & $N_d$ & 1 & 1 & 1 & 1 & 1 & 1 \\
& & $\delta_{\tt{scale}}$ & 0.1 & 0.1 & 0.1 & 0.1 & 0.1 & 0.1 \\
& & optimizer & adam, 0.9, 0.999 & adam, 0.9, 0.999 & adam, 0.9, 0.999 & adam, 0.9, 0.999 & adam, 0.9, 0.999 & adam, 0.9, 0.999 \\
& & learning rate $\alpha_{\theta}$ & 0.0003 & 0.0003 & 0.0003 & 0.0003 & 0.0003 & 0.0003 \\
\midrule
\multirow{10}{*}{policy} & \multirow{6}{*}{model} & network & mlp & mlp & mlp & mlp & mlp & mlp \\
& & $m_{\tt{fc}}$ & 1 & 1 & 1 & 2 & 2 & 2 \\
& & $m_{\tt{enc}}$ & 1 & 1 & 2 & 1 & 3 & 3 \\
& & activation & elu & elu & elu & elu & elu & elu \\
& & $d_h$ & 256 & 256 & 256 & 256 & 64 & 64 \\
& & $d_{\epsilon}$ & 10 & 10 & 10 & 10 & 32 & 100 \\
\cmidrule{2-9}
& \multirow{4}{*}{learning} & $n_{\tt{perturb}}$ & 10 & 10 & 10 & 10 & 10 & 10 \\
& & optimizer & adam, 0.9, 0.999 & adam, 0.9, 0.999 & adam, 0.9, 0.999 & adam, 0.9, 0.999 & adam, 0.9, 0.999 & adam, 0.9, 0.999 \\
& & $\xi, \eta, \nu$ & 0.01, 0.1, 1.1 & 0.01, 0.01, 1.1 & 0.01, 0.01, 1.1 & 0.01, 0.01, 1.1 & 0.01, 0.1, 1.3 & 0.01, 0.1, 1.3 \\
& & learning rate $\alpha_{\phi}$ & 0.0003 & 0.0003 & 0.0003 & 0.0003 & 0.0003 & 0.0003 \\
\midrule
\multirow{6}{*}{Q-network} & \multirow{4}{*}{model} & network & mlp & mlp & mlp & mlp & mlp & mlp \\
& & $m_{\tt{fc}}$ & 2 & 2 & 2 & 2 & 2 & 2 \\
& & activation & relu & relu & relu & relu & relu & relu \\
& & $d_h$ & 256 & 256 & 256 & 256 & 256 & 256 \\
\cmidrule{2-9}
& \multirow{2}{*}{learning} & optimizer & adam, 0.9, 0.999 & adam, 0.9, 0.999 & adam, 0.9, 0.999 & adam, 0.9, 0.999 & adam, 0.9, 0.999 & adam, 0.9, 0.999 \\
& & learning rate $\alpha_{\omega}$ & 0.0003 & 0.0003 & 0.0003 & 0.0003 & 0.0003 & 0.0003 \\
\midrule
\multirow{6}{*}{general} &  & $\alpha$ & 0.05 & 0.05 & 0.05 & 0.05 & 0.05 & 0.05 \\
& & $\tau$ & 0.005 & 0.005 & 0.005 & 0.005 & 0.005 & 0.005 \\
& & $\gamma$ & 0.99 & 0.99 & 0.99 & 0.99 & 0.99 & 0.99 \\
& & $n_Z$ & 100 & 10 & 100 & 100 & 10 & 10 \\
& & target calibration & no & no & yes & no & no & no \\
& & warm-up & 5000 & 10000 & 10000 & 10000 & 10000 & 10000 \\
\bottomrule
\end{tabular}
\caption{Hyperparameters for RL experiments.}
\label{table:rl-hyperparameters}
\end{table}
\vfill

\end{document}